\documentclass{article}

\usepackage{microtype}
\usepackage{graphicx}
\usepackage{subfigure}
\usepackage{booktabs} % for professional tables
\usepackage{amsthm}
\usepackage{amsmath}
\usepackage{amsfonts}

\usepackage{hyperref}

\newtheorem{theorem}{Theorem}
\newtheorem{lemma}[theorem]{Lemma}
\newtheorem{corollary}[theorem]{Corollary}

\newcommand{\reflemma}[1]{Lemma \ref{#1}}

\newcommand{\refcor}[1]{Corollary \ref{#1}}
\newcommand{\refeqn}[1]{Equation (\ref{#1})}
\newcommand{\reffig}[1]{Figure \ref{#1}}

\newcommand{\norm}[1]{\left\lVert#1\right\rVert}
\newcommand{\tv}{\mathbb{D}_\mathcal{TV}}
\newcommand{\rhoa}{\rho_{0:t}^{\alpha_t}}
\newcommand{\rhotot}{\rho_{0:t}(s_{0:t}, a_{0:t})}

\DeclareMathOperator*{\E}{\mathbb{E}}
\DeclareMathOperator*{\V}{\mathbb{V}}
\DeclareMathOperator*{\C}{\mathrm{Cov}}

\usepackage[accepted]{icml2020}

\icmltitlerunning{Policy Optimization Through Approximate Importance Sampling}

\begin{document}

\twocolumn[
\icmltitle{Policy Optimization Through Approximate Importance Sampling}



\icmlsetsymbol{equal}{*}

\begin{icmlauthorlist}
\icmlauthor{Marcin B. Tomczak}{cam}
\icmlauthor{Dongho Kim}{prowler}
\icmlauthor{Peter Vrancx}{prowler}
\icmlauthor{Kee-Eung Kim}{prowlerkaist}

\end{icmlauthorlist}

\icmlaffiliation{cam}{University of Cambridge (work partly done while at PROWLER.io)}
\icmlaffiliation{prowler}{PROWLER.IO}
\icmlaffiliation{prowlerkaist}{PROWLER.IO \& KAIST}

\icmlcorrespondingauthor{Marcin B. Tomczak}{mbt27@cam.ac.uk}

\icmlkeywords{Machine Learning, ICML}

\vskip 0.3in
]



\printAffiliationsAndNotice{}  % leave blank if no need to mention equal contribution

\begin{abstract}
Recent policy optimization approaches \citep{pmlr-v37-schulman15, DBLP:journals/corr/SchulmanWDRK17} have achieved substantial empirical successes by constructing new proxy optimization objectives. These proxy objectives allow stable and low variance policy learning, but require small policy updates to ensure that the proxy objective remains an accurate approximation of the target policy value. In this paper we derive an alternative objective that obtains the value of the target policy by applying importance sampling (IS). However, the basic importance sampled objective is not suitable for policy optimization, as it incurs too high variance in policy updates. We therefore introduce an approximation that allows us to directly trade-off the bias of approximation with the variance in policy updates. We show that our approximation unifies previously developed approaches and allows us to interpolate between them. We develop a practical algorithm by optimizing the introduced objective with proximal policy optimization techniques \citep{DBLP:journals/corr/SchulmanWDRK17}. We also provide a theoretical analysis of the introduced policy optimization objective demonstrating bias-variance trade-off. We empirically demonstrate that the resulting algorithm improves upon state of the art on-policy policy optimization on continuous control benchmarks.
\end{abstract}

\section{Introduction}

Policy gradient algorithms have achieved significant successes in reinforcement learning problems. Especially in continuous action settings policy gradient based methods provided a major milestone in achieving good empirical performance \citep{lillicrap2015continuous, Duan:2016:BDR:3045390.3045531}. Despite these results policy gradient approaches can still have significant drawbacks. The policy updates often suffer from high variance, which can result in requiring prohibitively large numbers of interactions with environment. Additionally, policy gradient algorithms require careful tuning of the update step size which can be difficult in practice. 

Recent policy optimization algorithms \citep{pmlr-v37-schulman15, DBLP:journals/corr/SchulmanWDRK17} have led to substantially improved the sample efficiency by optimizing a biased surrogate objective that has low variance. Optimizing this objective has been shown to stabilize learning \citep{Duan:2016:BDR:3045390.3045531, pmlr-v37-schulman15, pmlr-v70-achiam17a}. The bias incurred by using the surrogate objective can be controlled by restricting the divergence between the target and behavior policy. The algorithms derived in \cite{pmlr-v37-schulman15, DBLP:journals/corr/SchulmanWDRK17} perform small steps in policy space using the biased proxy objective.

In this paper, we derive a new policy optimization objective starting from the value of the target policy obtained through importance sampling (IS). Importance sampling  provides an unbiased estimate of the target policy's value using samples from the current policy. Unfortunately, the raw IS objective is a poor target for optimization as the variance of importance sampling can increase exponentially with the horizon. We therefore introduce an approximation that allows us to directly trade-off the bias of approximation with the variance in policy updates. We show that our approximation unifies the previous proxy optimization approaches with the pure importance sampling objective and allows us to interpolate between them. We demonstrate that the resulting algorithm significantly improves upon state of the art on-policy policy optimization on a number of continuous control benchmarks and exhibits more robustness to suboptimal choice of hyperparameters.

In addition to the empirical results, we also analyse the theoretical properties of the introduced policy optimization objective in terms of providing upper bounds on bias and variance. There are important gaps in our understanding of TRPO \citep{pmlr-v37-schulman15} and similar algorithms as only loose bounds on the introduced bias were provided. We also aim to analyse and extend the theory behind these algorithms by quantifying the exact error term for surrogate objective as opposed to merely providing an upper bound. The results we derive allow to obtain the results provided in several previous works as special cases, including \citet{pmlr-v37-schulman15, pmlr-v70-achiam17a, NIPS2017_6974}. Additionally, in the supplementary material we demonstrate that various policy gradient theorems \citep{Sutton:1999:PGM:3009657.3009806, pmlr-v32-silver14, ciosek2018expected} can be unified and proven by simply differentiating the equality we derive.

This paper is organized as follows. We establish the notation in the background section. We follow by presenting the related work and formerly obtained results. Next, we introduce a novel policy optimization algorithm interpolating between importance sampling and a previously used proxy policy optimization objective. We proceed by theoretical analysis of the introduced policy optimization objective and analyze its bias and variance. We conclude by presenting the empirical results achieved by the introduced algorithm on continuous control benchmarks. We provide implementation of the algorithm derived in this paper used to carry out experiments. \footnote{\url{https://github.com/marctom/POTAIS}}

\section{Background}
\label{sec:bg}
We begin by estabilishing notation. We assume a standard MDP formulation $(S, A, p, r, \gamma),$ with $S$ the set of states; $A$ the set of
 actions; $p:S\times A \times S \rightarrow  \mathbb{R}_{+}$ represents the transition model, where $p(s'|s, a)$ is the probability of transitioning to state $s'$ when $a$ is taken in $s$; $r:S \times A \rightarrow [-R_{max}, R_{max}]$ is the reward function and  $\gamma \in (0,1)$ is a discount factor. A policy $\pi: S \times A \rightarrow \mathbb{R}_{+}$ induces a Markov chain with transition matrix  $P_\pi$, i.e $P_\pi (s',s) = \int p(s'|s, a) \pi(a|s) da$.
We use $\tau$ to denote a state action trajectory $(s_0, a_0, s_1, a_1, \ldots)$. We use notation $\tau \sim \pi$ to denote that the trajectory is sampled by following policy $\pi$, i.e. $s_0 \sim \mu_0(\cdot)$, $a_t \sim \pi(\cdot|s_t)$ and $s_{t+1} \sim p(\cdot|s_t, a_t)$, where $\mu_0(\cdot)$ denotes the initial distribution over states. The value of a policy $\pi$ is defined by: 
\begin{equation}
    \eta(\pi) = \E_{\tau \sim \pi} \sum_{t \ge 0} \gamma^t r(s_t,a_t).
\end{equation}

By $\E_{\tau \sim \pi |s,a}$ we denote an expectation taken over the trajectories obtained by following policy $\pi$ that begin with state $s$ and action $a$, i.e. $\E_{\tau \sim \pi |s,a} \sum_{t \ge 0} \gamma^t f(s_t, a_t) =  \E_{\tau \sim \pi} \big [\sum_{t\ge 0} \gamma^t f(s_t, a_t) | s_0 = s, a_0 = a \big ]$. 
We use standard notation to denote the value function given state $s$, $V^\pi (s)= \E_{\tau \sim \pi} \big[ \sum_{t\ge 0} \gamma^t r(s_t,a_t) |s_0 = s \big]$, and similarly we denote state action value $Q^\pi(s,a)= \E_{\tau \sim \pi} \big[\sum_{t\ge 0} \gamma^t r(s_t,a_t) |s_0 = s, a_0 = a \big]$. The advantage of policy $\pi$ at state $s$ and action $a$ is defined as $A^\pi (s,a) = Q^\pi(s,a) - V^\pi(s)$.
Given the initial distribution over states $\rho_0$, we denote normalized discounted state occupancy measure by $d^\pi = (1-\gamma)\sum_{t \ge 0} \gamma^t P_\pi^t \mu_0$, so it follows that $d^{\pi}(s) = (1-\gamma) \sum_{t\ge 0} \gamma^t p(s_t = s| s_0 \sim \mu_0)$. 
We have that $\frac{1}{1-\gamma} \E_{s \sim d^\pi, a \sim \pi(\cdot|s)} f(s, a) =  \E_{\tau \sim \pi} \sum_{t \ge 0} \gamma^t f(s_t, a_t) $, where RHS can be estimated by performing rollouts from policy $\pi$.

We use $\rho_t(s_t, a_t) = \frac{\tilde\pi(a_t|s_t)}{\pi(a_t|s_t)}$ to denote importance sampling ratios and $\rho_{0:t}(s_{0:t}, a_{0:t}) = \prod_{i=0}^t \frac{\tilde\pi(a_i|s_i)}{\pi(a_i|s_i)}$ to denote their product. Throughout the paper we use $\tilde\pi$ to denote the target policy we aim to optimize and $\pi$ the behavior policy that was used to gather the current batch of data.
Lastly, we denote the total variation distance between two distributions $p$ and $q$ as $\tv\big(p(\cdot)||q(\cdot)\big):= \frac{1}{2} \int |p(x) - q(x)| dx$. 

\section{Prior work}
\label{sec:related}
In this section we state well-known results linking the value of target policy $\eta(\tilde\pi)$ with the value of data gathering policy $\eta(\pi)$.
\citet{Kakade:2002:AOA:645531.656005} establish the expression for the difference of values between two arbitrary policies $\tilde\pi$ and $\pi$. They show that for any two policies $\pi$ and $\tilde\pi$, the difference of values $\eta(\tilde \pi) - \eta(\pi)$ can be expressed as:
\begin{equation}
\label{eqn:lem_kakade}
\eta(\tilde \pi) - \eta(\pi) = \frac{1}{1 - \gamma}\E_{s \sim d^{\tilde\pi}, a \sim \tilde \pi(\cdot|s)}  A^\pi (s,a).
\end{equation}

Note \refeqn{eqn:lem_kakade} cannot be used directly for policy optimization  as the RHS of \refeqn{eqn:lem_kakade} requires samples from $d^{\tilde\pi}$, the discounted state occupancy measure of the target policy $\tilde\pi$. Following \citet{Kakade:2002:AOA:645531.656005, pmlr-v37-schulman15}, we can replace the occupancy measure $d^{\tilde\pi}$ with $d^{\pi}$ to define an approximation $L_\pi(\tilde\pi)$:
 
 \begin{equation}
 L_{\pi}(\tilde\pi) = \frac{1}{1-\gamma} \E_{s \sim d^\pi, a \sim \tilde \pi(\cdot|s)} A^\pi(s,a),
 \end{equation}

where we have slightly departed from the notation in \citet{pmlr-v37-schulman15}, by using $L_{\pi}(\tilde\pi)$ to denote the proxy for the difference $\eta(\tilde\pi) - \eta(\pi)$ instead of the value $\eta(\tilde\pi)$.

The quantity $L_{\pi}(\tilde\pi)$ can be estimated in practice and forms a backbone for modern policy optimization algorithms \citep{pmlr-v37-schulman15, DBLP:journals/corr/SchulmanWDRK17}. However, changing the occupancy measure in \refeqn{eqn:lem_kakade} introduces error which needs to be quantified. \citet{Kakade:2002:AOA:645531.656005} provided a lower bound on the value of policy being a linear combination of a target policy and behavior policy, $\forall_{s \in S} \ \tilde \pi (\cdot|s)  = (1-\alpha) \pi(\cdot|s) + \alpha \pi' (\cdot|s)$. They demonstrate that given behavior policy $\pi$, target policy $\pi'$, and their linear combination $\tilde \pi (\cdot|s) := (1-\alpha) \pi(\cdot|s) + \alpha \pi' (\cdot|s)\  \forall s \in S$, the following bound holds for $\alpha \in [0,1]$:

\begin{equation}
\eta(\tilde\pi) - \eta(\pi) \ge \alpha L_\pi(\pi')- \frac{2\alpha^2\epsilon\gamma}{(1-\gamma)(1- \gamma(1-\alpha))},
\end{equation}
where $\epsilon = \max_{s \in \mathcal{S}} |\E_{a \sim \tilde\pi} A^\pi (s,a)|$.

\citet{pmlr-v37-schulman15} follow on these results by providing the bound that is valid for any policies $\tilde\pi$ and $\pi$:
\begin{equation}
\label{eq:schulman_bound}
\eta(\tilde\pi) - \eta(\pi) \ge L_{\pi}(\tilde \pi) - \frac{4 \epsilon \gamma \Delta^2}{(1-\gamma)^2},
\end{equation}

where we denote $\Delta = \max_{s\in S} \tv \big (\pi(\cdot|s)|| \tilde\pi(\cdot|s)\big )$ and $\epsilon = \max_{s \in \mathcal{S}} |\E_{a \sim \tilde\pi} A^\pi (s,a)|$. The practical version of the algorithm derived in \citet{pmlr-v37-schulman15} resorts to maximising $L_{\pi}(\tilde\pi)$ w.r.t. $\tilde\pi$ in a neighbourhood of $\pi$, as the authors note that the obtained bound is too loose for the practical use. \citet{pmlr-v70-achiam17a} further improves the bound given \refeqn{eq:schulman_bound} by replacing maximum operator with expectation. They show that the for any policies $\tilde\pi$ and $\pi$ the following bound holds:
\begin{equation}
\eta(\tilde\pi) - \eta(\pi) \ge L_{\pi}(\tilde \pi) - \frac{2 \epsilon \gamma \delta}{(1-\gamma)^2},
\end{equation}
where we denote $\delta := \E_{s\sim d^\pi} \tv \big (\pi(\cdot|s) || \tilde\pi(\cdot|s)\big)$ and $\epsilon = \max_{s \in \mathcal{S}} |\E_{a \sim \tilde \pi(\cdot|s)} A^\pi (s,a)|$. The presented bounds indicate that $L_\pi(\tilde\pi)$ becomes a good proxy for $\eta(\tilde\pi) - \eta(\pi)$ when policies $\pi$ and $\tilde\pi$ are close in terms of some form of divergence. The benefit of introducing $L_\pi(\tilde\pi)$ is that it provides a biased but low variance surrogate for $\eta(\tilde\pi) - \eta(\pi)$. The bias can be controlled by restricting the proximity of $\tilde\pi$ and $\pi$. Algorithms optimizing $L_\pi(\tilde\pi)$ subject to some constraints have seen substantial empirical success \citep{pmlr-v37-schulman15, DBLP:journals/corr/SchulmanWDRK17, Duan:2016:BDR:3045390.3045531}.

Importance sampling has a long history of begin applied in off-policy evaluation \citep{pmlr-v48-thomasa16, Precup00eligibilitytraces} and also used in context of PAC bounds for both off-policy evaluation \citep{Thomas:2015:HCO:2888116.2888134} and policy optimization \citep{pmlr-v37-thomas15}. \citet{NIPS2018_7789} optimize the value of the policy $\eta(\tilde\pi)$ obtained via importance sampling while regularizing divergence between $\tilde\pi$ and $\pi$ to obtain low variance in the estimated value $\eta(\tilde\pi)$. Lowering the variance of IS estimator has been considered for value function estimation from off-policy data \cite{mahmood15, retrace, pmlr-v80-espeholt18a}.

Multiple approaches employ replay buffer \citep{mnih2013playing} to learn state action value function \citep{pmlr-v32-silver14, td3} to improve sample efficiency of learning. Although this approach can lead to good empirical performance \citep{Duan:2016:BDR:3045390.3045531} it can also suffer from stability and possibly diverge  by simultaneously applying bootstraping, off-policy learning and function approximation \citep{rl_book, Hasselt2018DeepRL}. Additionally using replay buffer significantly complicates theoretical analysis as introduced biases in general depend on all policies supplying data to buffer. In this work we focus on a setting in which only data incoming from current policy $\pi$ is used to learn target policy $\tilde\pi$.

\section{Approximate Importance Sampling}
\label{sec:approx_is}

\begin{table*}[ht]
\hspace*{1mm}
\scriptsize
\begin{tabular}[th]{lcccc}
\hline
&$f(s_{0:t}, a_{0:t})$& $\alpha_t$ & bias & variance\\
\hline
IS & $\rhotot A^\pi(s_t,a_t)$ & $(1,1, \ldots, 1)$ & no & high \\
TRPO & $\frac{\tilde\pi(a_t|s_t)}{\pi(a_t|s_t)}A^\pi(s_t,a_t)$   &  $(0,0, \ldots, 0, 1)$ & yes & low  \\
PPO & $ \min(\phi^\omega \big (\frac{\tilde\pi(a_t|s_t)}{\pi(a_t|s_t)}\big ) A^\pi(s_t, a_t), \frac{\tilde\pi(a_t|s_t)}{\pi(a_t|s_t)} A^\pi(s_t, a_t))$   &  - & yes & low  \\
This paper & $ \min(\phi^\omega \big (\prod_{i=1}^t \big ( \frac{\tilde\pi(a_i|s_i)}{\pi(a_i|s_i)} \big ) ^{\alpha_t^i} \big ) A^\pi (s_t, a_t), \prod_{i=1}^t \big ( \frac{\tilde\pi(a_i|s_i)}{\pi(a_i|s_i)} \big ) ^{\alpha_t^i} A^\pi(s_t, a_t))$   & any $\alpha_t \in [0,1]^{t+1}$ & depends on $\alpha_t$ & depends on $\alpha_t$ \\
\hline
\end{tabular}

\caption{Comparison of various functions used to approximate the value of target policy, $\eta(\tilde\pi)$. Importance sampling is unbiased, but introduces excessive variance. TRPO reduces the variance by
truncating products of importance sampling ratios, but this introduces bias. PPO provides a stable optimization procedure to optimize $L_\pi(\tilde\pi)$ in the proximity of $\pi$. Our objective allows to interpolate between TRPO and IS and in turn trade off bias and variance.}
\label{tab:caption}
\end{table*}

Recent policy optimization approaches \citep{pmlr-v37-schulman15, DBLP:journals/corr/SchulmanWDRK17} optimize the proxy objective $L_\pi(\tilde\pi)$ by keeping $\tilde\pi$ in proximity of $\pi$ where $L_\pi(\tilde\pi)$ remains an accurate approximation.  Alternatively, the value of target policy $\eta(\tilde\pi)$ can be obtained by applying importance sampling. 

Given a function $f:S \times A \rightarrow \mathbb{R}$ and two policies $\tilde\pi$ and $\pi$, the expectation $\E_{\tau \sim \pi} \sum_{t\ge 0} \gamma^t f(s_t, a_t)$ can be estimated with the step-based IS estimator \citep{Precup00eligibilitytraces, pmlr-v48-jiang16}:
\begin{equation}
    \label{eqn:is}
    \E_{\tau \sim \tilde\pi} \sum_{t\ge 0} \gamma^t f(s_t, a_t) = \E_{\tau \sim \pi} \sum_{t \ge 0} \gamma^t  \rhotot f(s_t, a_t),
\end{equation}

By applying the equality from \refeqn{eqn:is}  to change the discounted occupancy measure in the RHS of \refeqn{eqn:lem_kakade} from $d^{\tilde\pi}$ to $d^\pi$, we get: 
\begin{equation}
\label{eqn:kakade_is}
    \eta(\tilde\pi) - \eta(\pi) = \E_{\tau \sim \pi} \sum_{t \ge 0 } \gamma^t \rhotot A^\pi(s_t, a_t).
\end{equation}

The RHS of \refeqn{eqn:kakade_is} can be estimated with samples as the expectation is taken over trajectories coming from current behavior policy $\pi$. However, \refeqn{eqn:kakade_is} is not suitable for policy optimization algorithms as the variance of importance sampling ratio products $ \rhotot $ increases exponentially with the horizon \citep{pmlr-v48-jiang16, retrace}. In this section we introduce an objective function akin to $L_\pi(\tilde\pi)$ that allows to trade-off the bias of approximation with the variance in policy updates.

We begin with the observation that the definition of $L_\pi(\tilde\pi)$ can be seen as approximating products of importance sampling weights $ \rhotot$ with only the last term, $ \rhotot \approx \frac{\tilde\pi(a_t|s_t)}{\pi(a_t|s_t)}$. This approximation significantly reduces variance, but also introduces bias. Thus TRPO defines: 

\begin{equation}
   f^{TRPO}(s_{0:t}, a_{0:t}) := \frac{\tilde\pi(a_t|s_t)}{\pi(a_t|s_t)}A^\pi(s_t,a_t), 
\end{equation}
while the IS estimator defines:

\begin{equation}
    f^{IS}(s_{0:t}, a_{0:t}):= \rhotot A^\pi(s_t,a_t).
\end{equation}

 In both cases, $\eta(\tilde\pi) - \eta(\pi)$ is then estimated by the expectation $\E_{\tau \sim \pi} \sum_{t \ge 0} \gamma^t f(s_{0:t}, a_{0:t})$. 
 This can be seen as resorting to two extremes: constructing a biased estimator with low variance or an unbiased estimator with high variance. To unify these approaches and interpolate between them we introduce the following function $f^\alpha$:
\begin{equation}
    \label{eqn:f_def}
    f^{\alpha}(s_{0:t}, a_{0:t}) := \prod_{i=1}^t \Bigg ( \frac{\tilde\pi(a_i|s_i)}{\pi(a_i|s_i)} \Bigg ) ^{\alpha_t^i} A^\pi (s_t, a_t),
\end{equation}
where $\alpha_t$ are vectors of length $t+1$ with coordinates $\alpha^i_t \in [0,1]$. Using \refeqn{eqn:f_def} allows us to construct the following approximation of $\eta(\tilde\pi) - \eta(\pi)$:
\begin{equation}
\label{eqn:l_def}
L_\pi^\alpha(\tilde\pi) = \E_{\tau \sim \pi} \sum_{t \ge 0 } \gamma^t \prod_{i=1}^t \Bigg ( \frac{\tilde\pi(a_i|s_i)}{\pi(a_i|s_i)} \Bigg ) ^{\alpha_t^i} A^\pi (s_t, a_t).
\end{equation}
Note that using $\alpha_t = (0,0, \ldots, 0, 1)$, $\forall t \ge 0$ corresponds to definition of $L_\pi(\tilde\pi)$ and using $\alpha_t = (1,1, \ldots, 1)$, $\forall t \ge 0$ recovers importance sampling. 
Intermediate values of $\alpha_t$ will trade off bias and variance. To see this we take the weighted power mean of $\rho_t(s_t,a_t) = \frac{\tilde{\pi}(a_t|s_t)}{\pi(a_t|s_t)}$ and 1 (i.e. $\tilde{\pi}(a_t|s_t) \approx \pi(a_t|s_t))$ with respect to weight $\alpha_t^i$ so that 
\begin{equation*}
    \hat{\rho}_t(s_t,a_t) = \rho_t(s_t,a_t)^{\alpha_t^i} \cdot 1^{(1-\alpha_t^i)}
    = \rho_t(s_t,a_t)^{\alpha_t^i},
\end{equation*}
for all $t$. Hence,
\begin{equation}
     \rho_{0:t}^{\alpha_t}(s_{0:t}, a_{0:t}) := \prod_{i=1}^t \Bigg ( \frac{\tilde\pi(a_i|s_i)}{\pi(a_i|s_i)} \Bigg ) ^{\alpha_t^i},
\end{equation}
which leads to the definition of function $f^\alpha$ in \refeqn{eqn:f_def}.
Note that when $\alpha \rightarrow 0$, $ f^{\alpha}(s_{0:t}, a_{0:t}) $ converges to a constant independent of $\tilde\pi$. This gives an estimator with large bias, but no variance. In fact, when $\alpha_t=(0,0,\ldots,0)$ for every $t\ge 0$ then
\begin{equation}
    \E_{a_t \sim \pi (\cdot|s_t)} f^{\alpha_t}(s_{0:t}, a_{0:t})  = \E_{a_t \sim \pi (\cdot|s_t)} A^\pi (s_t,a_t) = 0.
\end{equation}

Using $\prod_{i=1}^t \big ( \frac{\tilde\pi(a_i|s_i)}{\pi(a_i|s_i)} \big )^{\alpha_t^i} = e^{\sum_{i=1}^t \alpha_t^i \log \rho_{i}(s_i, a_i)}$ can be also viewed as temperature smoothing \citep{Kirkpatrick671} with temperatures $T_i=(\alpha_t^i)^{-1}$. High temperatures cause terms $ e^{\frac{\log\rho_i(s_i, a_i)}{T_i}}$ to be uniform and do not influence policy optimization.

To obtain further insights we analyse the gradient of $f^{\alpha_t} (s_{0:t}, a_{0:t})$ with respect to the parameters of policy $\tilde\pi$:
\begin{align}
\label{eqn:grad}
    &\nabla f^{\alpha_t} (s_{0:t}, a_{0:t}) = \nabla e^ {\sum_{i=0}^t \alpha_t^i \log \frac{\tilde\pi(a_i|s_i)}{\pi(a_i|s_i)}} A^\pi (s_t, a_t)  \nonumber\\  
    & = e^ {\sum_{i=0}^t \alpha_t^i \log \frac{\tilde\pi(a_i|s_i)}{\pi(a_i|s_i)}} \nabla \sum_{i=0}^t \alpha_t^i \log \frac{\tilde\pi(a_i|s_i)}{\pi(a_i|s_i)}  A^\pi (s_t, a_t) \nonumber \\ 
    & = \prod_{i=0}^t \rho_{i}(s_i, a_i)^{\alpha_t^i} \sum_{i=0}^t \alpha_t^i \frac{\pi(a_i|s_i)}{\tilde\pi(a_i|s_i)} \nabla \tilde\pi(a_i|s_i)  A^\pi (s_t, a_t).
\end{align}
Since importance sampling ratios $\rho_t(s_t) = 1$ when $\tilde\pi = \pi$, evaluating the gradient $\nabla f^{\alpha_t} (s_{1:t}, a_{1:t})$ at $\tilde\pi = \pi$ results in
\begin{equation}
    \nabla f^\alpha (s_{0:t}, a_{0:t}) \Big |_{\tilde\pi = \pi} =  \sum_{i=0}^t \alpha_t^i \nabla \tilde\pi(a_i|s_i)  A^\pi (s_t, a_t).
\end{equation}
As the expectation on RHS of \refeqn{eqn:l_def} is taken over trajectories from policy $\pi$ which does not depend on $\tilde\pi$, the gradient of the objective $L_\pi^\alpha(\tilde\pi)$ evaluated at $\tilde\pi = \pi$ reads
\begin{equation}
    \nabla L_\pi^\alpha (\tilde\pi)\Big |_{\tilde\pi=\pi}  = \E_{\tau \sim \pi} \sum_{t \ge 0}\gamma^t \sum_{i=0}^t \alpha_t^i \nabla \tilde\pi(a_i|s_i)  A^\pi (s_t, a_t),
\end{equation}
which can be rewritten as:
\begin{equation}
\label{eqn:obj_grad}
    \nabla L_\pi^\alpha (\tilde\pi)\Big |_{\tilde\pi=\pi} = \E_{\tau \sim \pi} \sum_{t \ge 0} \gamma^t \nabla\tilde\pi(a_t|s_t) \sum_{i \ge t}  \gamma^{i} \alpha^t_i A^\pi (s_i, a_i).
\end{equation}
Since we have that the telescopic sum $\sum_{i = t}^T \gamma^i A^\pi(s_t, a_t) =  \sum_{i=t}^T \gamma^i r(s_t,a_t) + \gamma^{T+1}V(s_{T+1}) -  \gamma^t  V^\pi(s_t)$, in the case where $\tilde\pi = \pi$ the introduced weighting by $\alpha_t$ can be viewed as interpolating between using Monte Carlo returns and value function $V^\pi(s_t)$ to estimate the returns for policy gradient. Also, note that in the case $\tilde\pi = \pi$, we have that $\nabla L_\pi^\alpha (\tilde\pi)\big |_{\tilde\pi=\pi}  = \nabla \eta(\tilde\pi)|_{\tilde\pi=\pi}$ for any selection of $\alpha_t$. We extend this reasoning to a setting in which $\tilde\pi \ne \pi$.

Smoothing by $\alpha_t$ modifies the policy optimization target $L^\alpha_\pi(\tilde\pi)$. During policy optimization we also need to consider the domain of policy $\tilde\pi$ over which we optimize as optimization can attempt to set individual importance sampling ratios $\rho_t$ to extreme values.
Hence we are still required to constrain the divergence $\tv(\pi(\cdot|s) || \tilde \pi(\cdot|s))$ during optimization to prevent $\tilde\pi$ from diverging.
To optimize $L_\pi^\alpha(\tilde\pi)$ in a stable manner we adopt the clipping scheme from \citet{DBLP:journals/corr/SchulmanWDRK17}. We clip the values of $\prod_{i=1}^t \big ( \frac{\tilde\pi(a_i|s_i)}{\pi(a_i|s_i)} \big ) ^{\alpha_t^i}$ to be within the range of $(1-\omega, 1+\omega )$ and take the minimum with $f^\alpha(s_{0:t},a_{0:t})$ to ensure the resulting function is a lower bound to $f^\alpha(s_{0:t},a_{0:t})$. 
We view employing clipping as obtaining an approximate solution to the constrained optimization problem $\max_{\tilde\pi} L^\alpha_\pi(\tilde\pi)$ s.t. $\E_{\tau \sim \pi} \tv (\pi(\cdot|s) || \tilde \pi(\cdot|s)) < \epsilon$ which is necessary to execute optimization in a practical setting \citep{DBLP:journals/corr/SchulmanWDRK17}. 
While KL regularization can be used for this purpose, we select clipping as it is reported to yield better empirical performance \citep{DBLP:journals/corr/SchulmanWDRK17}. This results in the following definition of $ f^\alpha_{\phi^\omega}(s_{0:t}, a_{0:t})$:

\begin{align}
    f^\alpha_{\phi^\omega}(s_{0:t}, a_{0:t}) = \min \big [ & \phi^\omega(\hat \rho_{0:t}^{\alpha_t}) A^\pi(s_t, a_t), f^\alpha(s_{0:t}, a_{0;t}) \big],
\end{align}

where $\phi^\omega$ denotes a clipping function with the range of $(1-\omega, 1+\omega)$. Next, the corresponding policy optimization objective $L_\pi^{\alpha, \phi^\omega}$ is defined as
\begin{equation}
    L_\pi^{\alpha, \phi^\omega}(\tilde\pi) = \E_{\tau \sim \pi} \sum_{t \ge 0} \gamma^t f^{\alpha_t}_{\phi^\omega}(s_{0:t}, a_{0:t}).
\end{equation}
To obtain a practical algorithm we approximate $\nabla L_\pi^{\alpha, \phi^\omega}(\tilde\pi)$ with subsampled transitions. We collect the set of transitions $\{(s_t, a_t, r_t, s_{t+1} ) \}_{t=1}^T$ we subsample a minibatch of timesteps $B$, and approximate the gradient $\nabla L_\pi^{\alpha, \phi^\omega}(\tilde\pi) \approx  |B|^{-1} \sum_{t \in B} \gamma^t \nabla f^\alpha_{\phi^\omega}(s_{0:t}, a_{0:t})$. We use Generalized Advantage Estimation \citep{schulman2015high} to approximate the advantage function $A^\pi(s_t, a_t)$ with $\hat A^\pi (s_t, a_t)$. As the derived algorithm performs approximate importance sampling, we call it Approximate IS Policy Optimization. The resulting policy optimization procedure is summarized as Algorithm 1. We provide comparison of Approximate IS Policy Optimization with existing policy optimization approaches in Table 1.

Note that in general we face three sources of possible bias in updates of policy $\tilde\pi$ incoming from: (i) learning approximate critic $\hat A^\pi(s,a)$ using function approximation and bootstraping to lower the variance of policy updates \cite{schulman2015high, tomczak2019compatible}, (ii) clipping applied to define $f^\alpha_{\phi^\omega}(s_{0:t}, a_{0:t})$ to prevent individual importance weights $\rho_t$ from exploding during optimization and (iii) bias due to use of approximate occupancy measure $d^\pi$. PPO can control bias (i) by adjusting the $\lambda$ parameter in GAE critic and can control bias (ii) by increasing the clip value $\omega$ or reducing the number of optimization steps to learn policy $\tilde\pi$. However, PPO cannot correct for the bias incoming from employing $d^\pi$ as an occupancy measure. As importance sampling ratios are multiplied this type of bias can increase exponentially even if $\pi$ and $\tilde\pi$ remain close. The approach we develop enables us to control bias of type (iii) by adjusting $\alpha_t$. Therefore Approximate IS Policy Optimization can trade-off total bias in policy updates with their variance.

\section{Theoretical Analysis}

\label{sec:theoretical_analysis}

We begin by presenting theoretical analysis of the quality of the introduced estimator $L_\pi^\alpha(\tilde\pi)$ in terms of providing upper bounds on a bias and variance on its estimator. We defer all proofs and discussion of assumptions of the subsequent results to the Appendix. We drop the dependency on the $s_{0:t}$ and $a_{0:t}$ when it is clear from the context. We first analyse the properties of random variable $\hat L_{\pi}^\alpha(\tilde \pi): =    \sum_{t \ge 0} \gamma^t \rhoa (s_{0:t}, a_{0:t})  A^{\pi}(s_t, a_t)$.
We begin with the following Lemma providing the upper bound on $\mathbb{V} [\hat L_{\pi}^\alpha(\tilde \pi)]$ in terms of norm $\norm{\alpha_t}_1$.

\begin{lemma}
\label{lem:L_var}
Consider truncated horizon estimator defined by $\hat L_{\pi, T}^\alpha(\tilde \pi): =    \sum_{t = 0}^T \gamma^t \rhoa (s_{0:t}, a_{0:t})  A^{\pi}(s_t, a_t)$. 
Assume that $\forall_{s \in \mathcal{S}, a \in \mathcal{A}} \  \frac{\tilde \pi(a|s)}{\pi(a|s)} < C_\rho$. The variance $\mathbb{V}_{\tau \sim \pi}[ L_{\pi, T}^\alpha(\tilde \pi)]$ is upper bounded as
\begin{align}
    & \V_{\tau \sim \pi}[ L_{\pi, T}^\alpha(\tilde \pi)] \le \frac{1 - \gamma^{T+1}}{1-\gamma} \epsilon^2
    \sum_{t=0}^T \gamma^{t} C_\rho^{2\norm{\alpha_t}_1},
\end{align}
 where $\epsilon = \max_{s \in \mathcal{S},a \in \mathcal{A}} |A^\pi (s,a)|$.
\end{lemma}
The significance of \reflemma{lem:L_var} is that once sparse vectors $\alpha_t$ with uniformly bounded number of non-zero components $K$ are employed we have that $\forall t \ge 0 \ \norm{\alpha_t}_1 \le K$ and the variance $\V_{\tau \sim \pi}[ L_{\pi, T}^\alpha(\tilde \pi)]$ is guaranteed to be finite and uniformly bounded by $\frac{\epsilon^2 C_\rho^{2K}}{(1-\gamma)^2}$. Note that if introduced smoothing by $\alpha_t$ is not used we have that $\V_{\tau \sim \pi} [\sum_{t \ge 0} \gamma^t \rho_{0:t}] \ge \sum_{t \ge 0} \gamma^{2t} \V_{\tau \sim \pi}[ \rho_{0:t}]$ which requires
$\V_{\tau \sim \pi}[\rho_{0:t}] < \gamma^{-2t}$ for convergence.
Hence the variance of step importance sampling estimator can diverge to infinity \citep{retrace} as opposed to $L_{\pi, T}^\alpha(\tilde \pi)$ when appropriate $\alpha_t$ are chosen. We follow by quantifying the bias of estimator  $ \hat L_\pi^\alpha(\tilde\pi) $.

\begin{lemma}
\label{lem:L_bias}
The absolute bias of $ \hat L_\pi^\alpha(\tilde\pi) $ i.e. $\mathbb{B}[\hat L_\pi^\alpha(\tilde\pi)] := |\eta(\tilde\pi) - \eta(\pi) - \E_{\tau \sim \pi }\hat L_\pi^\alpha(\tilde\pi)|$ is upper bounded by:
\begin{align}
    &\mathbb{B}[\hat L_\pi^\alpha(\tilde\pi)] \le \epsilon  \E_{\tau \sim \tilde \pi} \sum_{t\ge 0} \gamma^t \big | \prod_{i=0}^t \rho_i^{1- \alpha_t^i } - 1 \big |, 
\end{align}
 where $\epsilon = \max_{s,a} |A^\pi (s,a)|$.
\end{lemma}

\reflemma{lem:L_bias} implies that $\alpha_t = (1,1, \ldots, 1)$ for all $t \ge 0$ the introduced estimator $ \hat L_\pi^\alpha(\tilde\pi)$ remains unbiased, but decreasing the values of $\alpha_t$ will introduce bias which is controlled by the expected dispersion of values of $ |\prod_{i=0}^t \rho_i^{1- \alpha_t^i}-1|$ increasing with $\alpha_t^i$. 
We see that \reflemma{lem:L_var} together with \reflemma{lem:L_bias} allows us to trade off bias and variance of estimator $ \hat L_\pi^\alpha(\tilde\pi)$ when coordinates of vectors $\alpha_t$ are interpolated from $0$ to $1$. Thus \reflemma{lem:L_var} and \reflemma{lem:L_bias} validate the usefulness of estimator $\hat L_{\pi}^\alpha(\tilde \pi)$ in the context of off-policy evaluation. Note that we do need estimate of $\hat L_{\pi}^\alpha(\tilde \pi)$ directly to learn policy $\tilde\pi$. To perform the policy update we require an estimate of $\nabla \hat L_{\pi}^\alpha(\tilde \pi)$. Thus we now switch focus to provide the bounds on variability of $\nabla\hat L_{\pi}^\alpha(\tilde \pi)$ and norm of bias of
$\nabla \hat L_{\pi}^\alpha(\tilde \pi)$ in terms of the values of $||\alpha_t||$.

\begin{algorithm}[tb]
   \caption{Approximate IS Policy Optimization}
\begin{algorithmic}
   \STATE {\bfseries Input:} Initial policy $\pi_0$\\
   \REPEAT
   \STATE Sample trajectories $\{ (s_i, a_i, r_i, s_{t+1}) \}_{i=1}^N$ using $\pi$.
   \STATE Set $\tilde\pi = \pi_i$.
   \STATE      Estimate $\hat A^{\tilde\pi}(s,a)$ on $\{ (s_i, a_i, r_i, s_{i+1}) \}_{i=1}^N \sim \pi$.
   \FOR{$i=1$; $i \le \text{num policy updates}$}
   \STATE     Subsample minibatch B:  $\{ (s_i, a_i, r_i, s_{i+1}) \}_{i \in B} $.
    \STATE $\nabla L_\pi^{\alpha, \phi^\omega}(\tilde\pi) \approx \frac{1}{|B|} \sum_{t \in B} \gamma^t \nabla f^\alpha_{\phi^\omega}(s_{0:t}, a_{0:t}).$
   \STATE  Update $\tilde\pi$ using $\nabla L_\pi^{\alpha, \phi^\omega}(\tilde\pi)$.
   \ENDFOR
   \STATE     Set $\pi_{i+1} = \tilde\pi$. 
   \UNTIL{ $\pi_i$ \text{ has not converged}}
\end{algorithmic}
\end{algorithm}

\begin{lemma}
\label{lem:nablaL_var}
Given two policies $\pi$ and $\tilde\pi$ assume that $\forall_{s \in \mathcal{S}, a \in \mathcal{A}} \  C_\rho < \frac{\pi(a|s)}{\tilde\pi(a|s)} < C^\rho$ , $\forall_{s \in \mathcal{S}, a \in \mathcal{A}} \ ||\nabla \tilde \pi (a|s) ||_2 < C^\partial$ and $\E_{\tau \sim \pi }\sum_{t\ge 0} \gamma^t (\rhoa)^2 < C^\gamma$. Then we have that:
\begin{align}
    & \mathrm{tr} \Big [ 
    \C_{\tau \sim \pi} \big [\nabla \hat L_{\pi}^\alpha(\tilde \pi), \nabla \hat L_{\pi}^\alpha(\tilde \pi) \big ] \Big ] \le \zeta_{\pi, \tilde\pi} \E_{\tau \sim \pi}   \sum_{t \ge 0} \gamma^t \norm{\alpha_t}_1^2 , \label{eqn:pie}
\end{align}
for some constant $\zeta_{\pi, \tilde\pi} \in \mathbb{R}$.
\end{lemma}

The assumption $\E_{\tau \sim \pi }\sum_{t\ge 0} \gamma^t (\rhoa)^2 < C^\gamma$ can be satisfied by uniformly bounding the number of non-zero components by $K$ for all $t \ge 0$ as in this case we again have $\norm{\alpha_t} \le K \ \forall t \ge 0$.
It follows that $\sum_{t\ge 0} \gamma^t (\rhoa)^2 \le \sum_{t\ge 0} \gamma^t C_\rho^{2\norm{\alpha}_1} = \frac{C_\rho^{2\norm{\alpha}_1}}{1-\gamma}$.
Note that \reflemma{lem:nablaL_var} allows to control the variability $\mathrm{tr} 
\C_{\tau \sim \pi} [\nabla \hat L_{\pi}^\alpha(\tilde \pi), \nabla \hat L_{\pi}^\alpha(\tilde \pi) ]$ by adjusting the coordinates of $\alpha_t$. It follows that we can guarantee lower variance when sparse vectors $\alpha_t$ are employed. The variability of gradients $\C_{\tau \sim \pi} [\nabla \hat L_{\pi}^\alpha(\tilde \pi), \nabla \hat L_{\pi}^\alpha(\tilde \pi) ]$ is zero when we consider $\alpha_t = 0 \ \forall t > 0$, however in this setting the objective $\hat L_{\pi}^\alpha(\tilde \pi)$ does not depend on policy $\tilde \pi$ and hence $\nabla \hat L_{\pi}^\alpha(\tilde \pi) = 0$.
 Nevertheless when values of $\alpha_t$ are low we can expect  $\nabla \hat L_{\pi}^\alpha(\tilde \pi)$ to exhibit low variance. From an intuitive point of view, smoothing introduced by $\alpha_t^i$ flattens the optimization objective $L_{\pi}^\alpha(\tilde \pi)$ and hence reduces the variability of gradients $\nabla \hat L_{\pi}^\alpha(\tilde \pi)$.
 Next we now provide the result quantifying the norm of bias of $\nabla \hat L_{\pi}^\alpha(\tilde \pi)$ introduced by applying smoothing by $\alpha_t$.

\begin{lemma}
\label{lem:nablaL_bias}
Under assumptions from \reflemma{lem:nablaL_var} the norm of bias $\norm{\E_{\tau \sim \pi} \nabla \hat L_{\pi}^\alpha(\tilde \pi)- \nabla \eta(\tilde\pi)}_2$ is upper bounded as follows:
\begin{align}
    & \norm{\E_{\tau \sim \pi} \nabla \hat L_{\pi}^\alpha(\tilde \pi)- \nabla \eta(\tilde\pi)}_2 \nonumber \\
    & \le \epsilon C^\partial \E_{\tau \sim \pi } \sum_{t \ge 0}  \gamma^t \Big [ \rhoa \norm{ \alpha_t -\mathbf{1}}_1 +  (t+1) |\rhoa  - \rho_{0:t} \big | \Big ] .
\end{align}
\end{lemma}

\reflemma{lem:nablaL_bias} quantifies the bias of $\nabla \hat L_{\pi}^\alpha(\tilde \pi)$ as the expected sum of two terms depending of the distances of smoothed IS ratios from true ones $|\rhoa  - \rho_{0:t} \big |$ and the $L_1$ distance of $\alpha_t$ from the vector of ones $\mathbf{1}$ scaled by $\rhoa$. Similarly as in the case of $\hat L_{\pi}^\alpha(\tilde \pi)$ decreasing coordinates $\alpha_t$ from $1$ to $0$ causes these quantities to be non-zero and introduces bias. \reflemma{lem:nablaL_var} and \reflemma{lem:nablaL_bias} provide a basis to obtain trade-off between the norm of bias and variability of $\nabla \hat L_{\pi}^\alpha(\tilde \pi)$ for coordinates of $\alpha_t$ ranging from $0$ to $1$ which we exploit to improve the sample efficiency of policy optimization. 

\subsection{The surrogate objective $L_{\pi}(\tilde \pi)$.}
Next we present the results allowing us to unify the previous work on policy improvement of \citet{Kakade:2002:AOA:645531.656005, pmlr-v37-schulman15, pmlr-v70-achiam17a, pmlr-v28-pirotta13} which analyzed surrogate objective $L_\pi(\tilde\pi)$. In the supplement we demonstrate how the previously derived results can be obtained from the following \reflemma{lem:mpie}. 

\begin{lemma}
\label{lem:mpie}
Given two policies $\pi$ and $\tilde\pi$ we can express the difference of values $\eta(\tilde\pi) - \eta(\pi) $ as
\begin{align}
\label{eqn:mpie}
&\eta(\tilde\pi) - \eta(\pi) =  L_{\pi}(\tilde \pi) + \nonumber\\
& \frac{1}{1 - \gamma} \E_{s\sim d^\pi} \int \Delta^Q_{\pi,\tilde\pi}(s,a) \big (\tilde\pi(a|s) - \pi(a|s) \big ) da,
\end{align}
where $\Delta^Q_{\pi,\tilde\pi}(s,a) = Q^{\tilde\pi}(s,a) - Q^\pi(s,a)$.
\end{lemma}

The term $\frac{1}{1 - \gamma} \E_{s\sim d^\pi} \int \big (\tilde\pi(a|s) - \pi(a|s) \big ) \big (Q^{\tilde\pi}(s,a) - Q^\pi(s,a)\big ) da$ in \refeqn{eqn:mpie} quantifies the difference between the surrogate function $L_{\pi}(\tilde\pi)$ typically used in policy optimization algorithms and the real difference between the values of target policy $\eta(\tilde\pi)$ and data gathering policy $\eta(\pi)$. 
The equality derived in \reflemma{lem:mpie} should be compared with equality derived in \reflemma{eqn:lem_kakade}.
We have that $\eta(\tilde\pi) - \eta(\pi) = \frac{1}{1-\gamma} \E_{s \sim d^{\tilde\pi}} A^\pi (s,a) =  \frac{1}{1-\gamma} \E_{s \sim d^{\pi}} A^\pi (s,a) + \frac{1}{1 - \gamma} \E_{s\sim d^\pi} \int \big (\tilde\pi(a|s) - \pi(a|s) \big ) \big (Q^{\tilde\pi}(s,a) - Q^\pi(s,a)\big )$. Thus the change of the occupancy measure from $d^{\tilde\pi}$ to $d^\pi$ in \refeqn{eqn:lem_kakade} leads to an additional difference term $\frac{1}{1 - \gamma} \E_{s\sim d^\pi} \int \big (\tilde\pi(a|s) - \pi(a|s) \big ) \big (Q^{\tilde\pi}(s,a) - Q^\pi(s,a)\big ) da$.

 When reward $R$ does not depend on actions it is straightforward to obtain the upper bound on $|\eta(\tilde\pi) - \eta(\pi)| \le R_{max} \norm{ d^{\tilde\pi}  - d^\pi }_1 \le \frac{2 R_{max}}{1-\gamma} \E_{s \sim d^\pi} \tv(\pi(\cdot|s),\tilde\pi(\cdot|s))$ using Holder's inequality and Lemma 3 in \citet{pmlr-v70-achiam17a}. 
 The results known in the literature can be viewed as providing tighter bounds on quantity $\langle R, d^\pi  - d^{\tilde\pi }\rangle $. 
 Since using Holder's inequality applied with $p=1$ and $q=\infty$ can be viewed as uniform bounding it is justified to expect that the difference term $\eta(\tilde\pi)-\eta(\pi) -  L_{\pi}(\tilde \pi)$ can be represented in the form $\frac{1}{1-\gamma} \E_{s \sim d^\pi} \int \Delta^Q_{\pi,\tilde\pi}(s,a) \big ( \tilde \pi(a|s) - \pi(a|s) \big ) da $ for some weighting function $\Delta^Q_{\pi,\tilde\pi}$. Different ways of upper bounding $\Delta^Q_{\pi,\tilde\pi}(s,a)$ will lead to different bounds. The difficulty is to establish the expression for weights $\Delta^Q_{\pi,\tilde\pi}(s,a)$ which is done by \reflemma{lem:mpie}. Calculating the RHS of \refeqn{eqn:mpie} requires the access to Q-function of target policy $Q^{\tilde\pi}(s,a)$. Unfortunately, $Q^{\tilde\pi}(s,a)$ is difficult to be directly estimated in a practical setting. Previous works \citep{pmlr-v37-schulman15, pmlr-v70-achiam17a} used an uniform upper bounding of $Q^{\tilde\pi}(s,a) - Q^\pi(s,a)$ over the whole horizon resulting in terms dependent on $\frac{1}{1-\gamma}$. 

\section{Experiments}
In this section we aim to empirically validate the performance of the introduced policy optimization objective $L_\pi^\alpha(\tilde\pi)$. We firstly examine bias and variance of $L_\pi^\alpha(\tilde\pi)$ on a toy problem. We follow by demonstrating that Approximate IS Policy Optimization outperforms PPO \cite{DBLP:journals/corr/SchulmanWDRK17} on standard Mujoco benchmarks. Then we investigate the effect of varying $\alpha$ in the objective $L^{\alpha, \phi^\omega}_\pi(\tilde\pi)$ on the sample efficiency of learning in standard continuous control benchmarks. We also test the robustness to suboptimal choice of hyperparameters. We compare sample efficiency of policies $\pi_1, \pi_2$ as the ratio of values of $\frac{\eta(\pi_1)}{\eta(\pi_2)}$ at the end of learning.

\begin{table*}[ht!]
\begin{center}
\tiny
\begin{tabular}{ccccccccc}
\toprule
$\alpha_t$ &  Standup &  Reacher &  Ant &  Hopper &   Humanoid & Swimmer &  Walker2d &  HalfCheetah \\

\midrule
$(0, \ldots, 0.0, 0.5, 0.5, 0.5, 1.0)$ &      $1.19\pm0.05$ &  $0.98\pm0.09$ &  $1.24\pm0.06$ &  $0.98\pm0.06$ &  $1.48\pm0.06$ &  $1.15\pm0.06$ &  $1.02\pm0.05$ &  $1.03\pm0.24$ \\
$(0, \ldots, 0.0, 0.5, 0.5, 1.0)$     &      $1.14\pm0.04$ &  $0.96\pm0.07$ &  $1.25\pm0.07$ &   $0.97\pm0.1$ &  $1.46\pm0.07$ &   $1.11\pm0.1$ &  $0.98\pm0.06$ &  $1.23\pm0.29$ \\
$(0, \ldots, 0.0, 0.5, 1.0)$     &      $1.06\pm0.02$ &   $0.9\pm0.06$ &  $1.15\pm0.12$ &   $1.01\pm0.1$ &  $1.37\pm0.09$ &  $1.14\pm0.07$ &  $1.01\pm0.06$ &  $1.16\pm0.27$\\
\midrule

\end{tabular}

\end{center}
\caption{Relative improvement over PPO \citep{DBLP:journals/corr/SchulmanWDRK17} for well-performing schemes of $\alpha_t$. Approximate IS Policy sampling provides an improvements of $15\%-50\%$ in sample efficiency on five out of eight tested environments.}
\label{tab:best_alpha_results}
\end{table*}

\textbf{NChain}
To investigate the bias variance trade-off obtained by introducing the objective $L_\pi^\alpha(\tilde\pi)$ we carry out analysis of off-policy evaluation on NChain \citep{Strens:2000:BFR:645529.658114} environment. We loop over different values of $\tilde\pi(right)$ and keep $\pi(right) = 0.5$. We estimate the value of $\eta(\tilde\pi)$ using $\hat L_\pi^\alpha(\tilde\pi)$ for $\alpha_t = (0,0,\ldots, 0, \beta, 1)$ for $\beta$ in range $[0,1]$ and data gathered with policy $\pi$. We calculate the estimators based on $5 \cdot 10^5$ trajectories. We set environment slip parameter to the default value of $0.2$, the number of states to $5$ and the discount factor $\gamma$ to $0.8$. We report bias, standard deviation and RMSE (Root Mean Squared Error) of tested estimators in \reffig{fig:nchain_results}. As expected, we observe that increasing $\beta$ reduces bias but increases variance. This results in the U-shaped RMSE curve caused by bias-variance decomposition. When the distance between $\tilde\pi$ and $\pi$ increases the introduced estimator $\hat L_\pi^\alpha(\tilde\pi)$ outperforms $\hat L_\pi(\tilde\pi)$ achieving lower RMSE. Note this experiment evaluates the quality of $\hat L_\pi^\alpha(\tilde\pi)$ being an off-policy estimator of $\eta(\tilde\pi) - \eta(\pi)$.

\begin{figure}[ht!]   
\hspace*{0.4cm}  
\includegraphics[width=75mm,height=75mm]{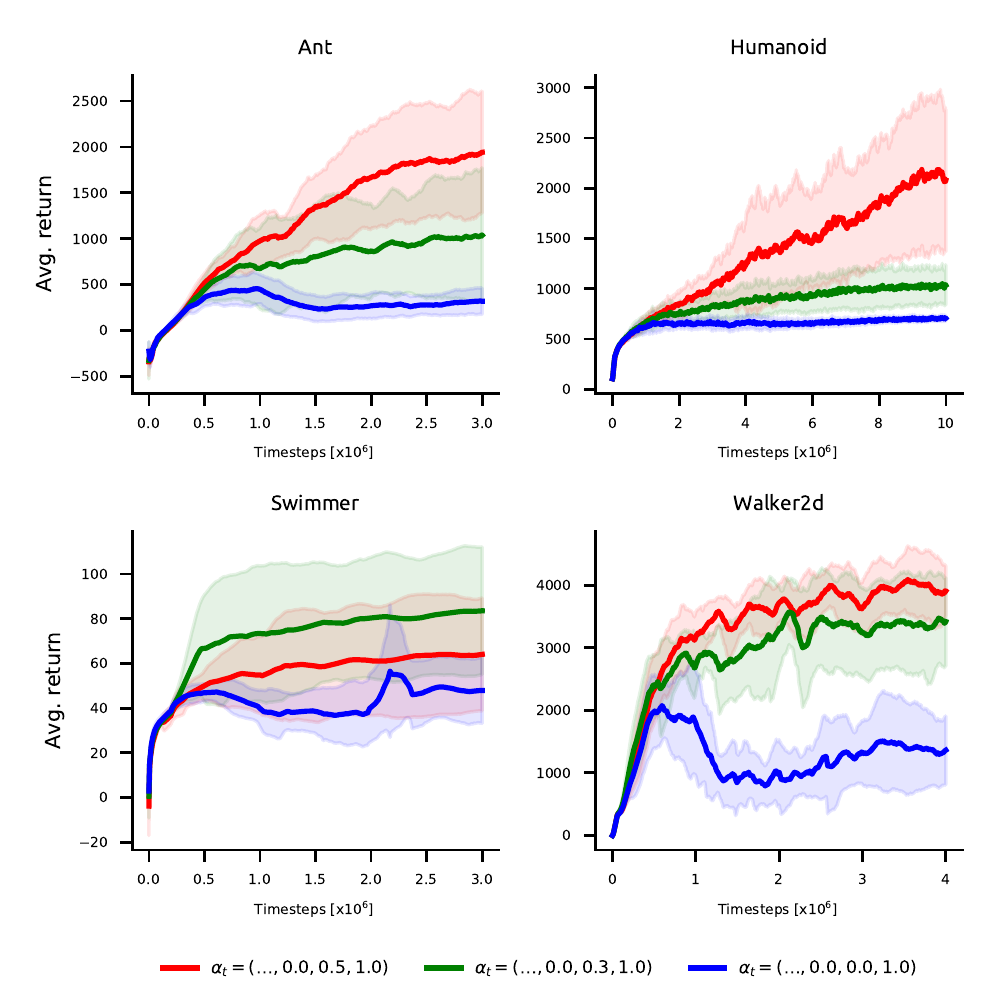}
\caption{Learning curves with standard deviation error bars for robustness test. We experiment with $\alpha_t = (0,0,\ldots, 0, \beta, 1)$ with $\beta$ in the range of $[0.0, 0.3, 0.5]$. Setting $\beta$ to values $0.3$ or $0.5$ leads to significantly better robustness to suboptimal hyperparameters than PPO (blue curve).}
\label{fig:robust_results}
\end{figure}

\begin{figure}[ht!]                                              
\includegraphics[width=80mm,height=40mm]{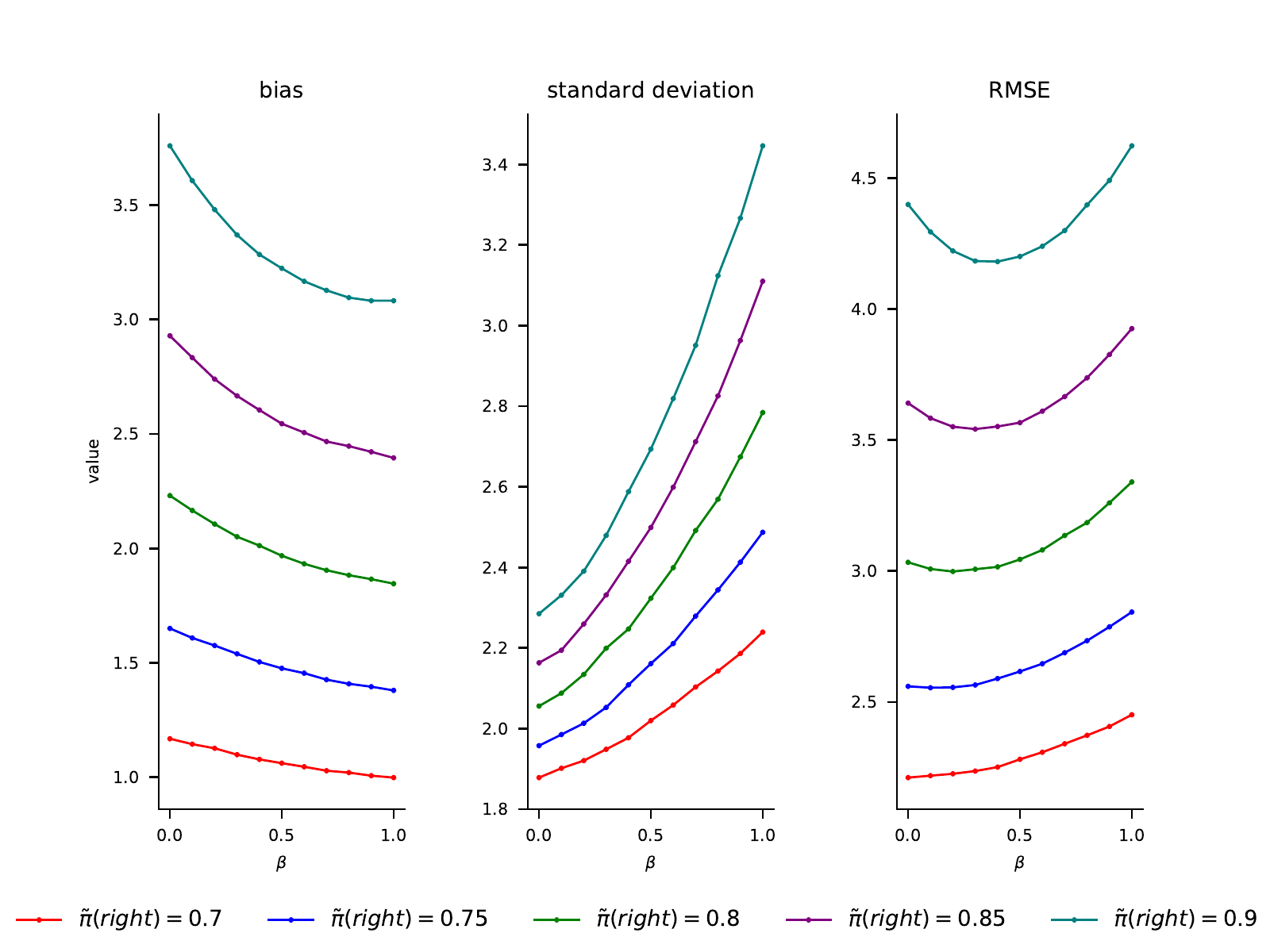}
\caption{The bias variance trade-off on NChain environment. We investigate $\alpha_t = (0,0,\ldots, 0, \beta, 1)$ for different $\tilde\pi(right) \in [0.7, 0.75, 0.8, 0.85, 0.9]$. Increasing $\beta$ from $0$ to $1$ reduces the bias but increases the variance of the estimator $L_\pi^\alpha(\tilde\pi)$.}
\label{fig:nchain_results}
\end{figure}

\textbf{Mujoco Environments}
 We follow by carrying out experimentation on high dimensional control problems using Mujoco \citep{mujoco} as a simulator. We adopt the experimental set up from \cite{pmlr-v37-schulman15, DBLP:journals/corr/SchulmanWDRK17} and experiment with eight representative Mujoco environments.

\begin{figure*}[ht!]
\includegraphics[width=170mm,height=61mm]{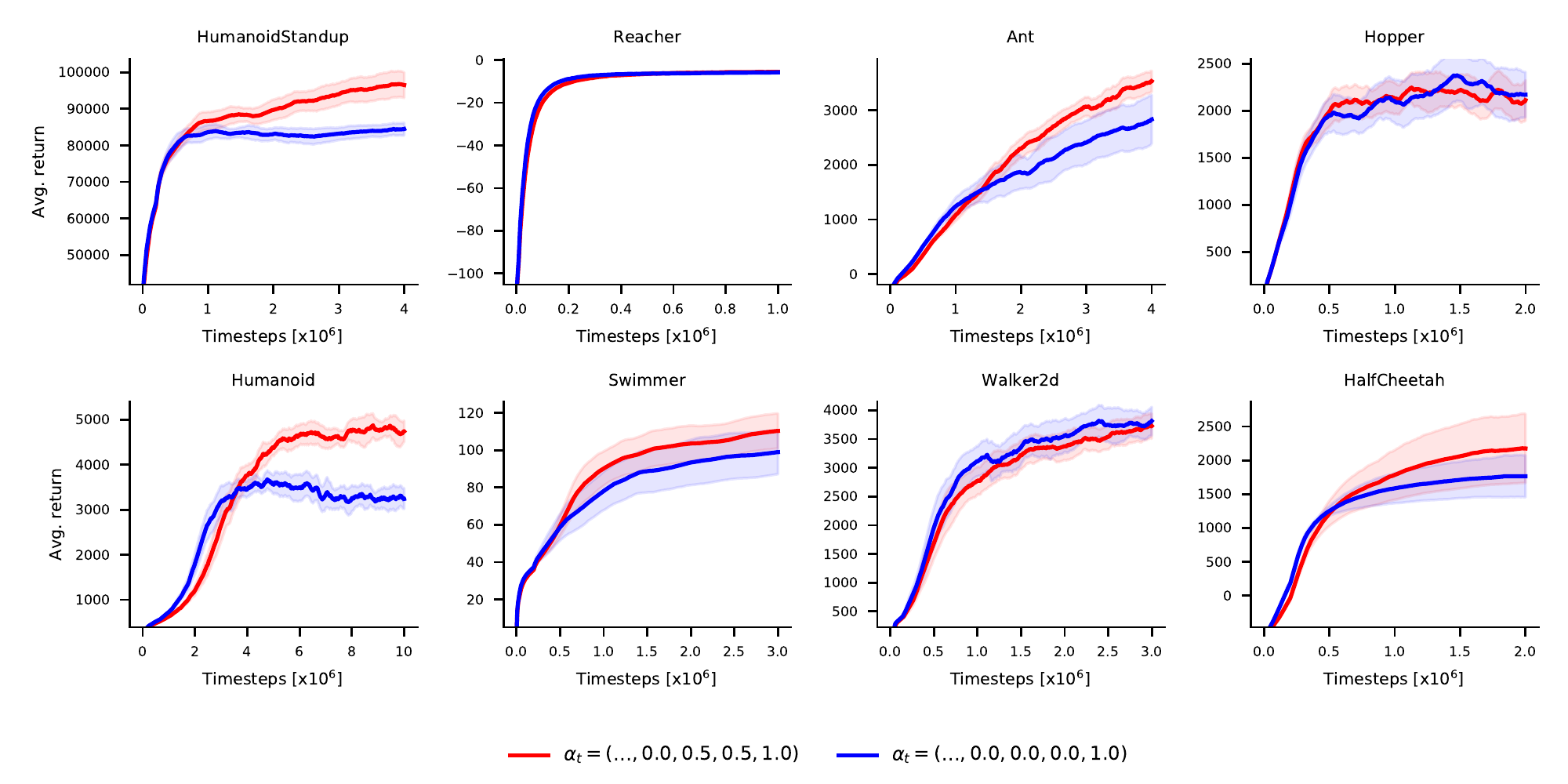}
\caption{Learning curves for continuous control benchmarks reported with half standard deviation error bars. We compare well-performing choice $\alpha_t= (0,0, \ldots, 0.0, 0.5, 0.5, 1)$ against PPO (blue curve). The results are averaged over $20$ seeds. Approximate IS Policy Optimization leads to an improved speed of learning on several environments. See text for full details.}
\label{fig:main_results}
\end{figure*}

In these experiments we use $\alpha_t = (0,0,\ldots, 0, 0.5, 0.5, 1)$. The benefit of these choice of $\alpha_t$ of policy optimization is that it incurs only modest additional computational burden and is straightforward to implement in a practical setting. Note that the choice $\alpha_t = (0,0,\ldots, 0, 0, 1)$ corresponds to PPO algorithm. We discuss the results of experimenting with a wide range of different settings of $\alpha_t$ later in the text.

 We average the learning curves over $20$ random seeds. The results are gathered in \reffig{fig:main_results}. We see that Approximate IS Policy Optimization increases speed of learning on five environments: HumanoidStandup, Ant, Swimmer, HalfCheetah and Humanoid and results in the same performance on three environments Walker2d, Hopper and Reacher. 
 We observe a substantial improvement corresponding to $48\%$ increase in sample efficiency on Humanoid environment and average improvement of $20 \%$ on four further environments (HumanoidStandup, Ant, Swimmer and HalfCheetah). We report detailed results and all experimental details allowing to replicate the experiments in the Appendix. Note Approximate IS Policy Optimization either outperforms PPO or results in a statistically insignificant difference in returns.

\textbf{Robustness test}
Next we test the robustness of Approximate IS Policy Optimization to suboptimal hyperparameters. In practice tuning hyperparameters per environment leads to significant computational burden and often suboptimal hyperparameters have to be used. For this task, we run the algorithm in four Mujoco environments and average the learning curves over $8$ separate seeds. During these experiments we use too high clip value, i.e. $0.4$ for Ant, Swimmer and Walker2d tasks and $0.2$ for Humanoid environment. In this task we again experiment with the choice $\alpha_t = (0,0,\ldots, 0, \beta, 1)$. We report the learning curves in \reffig{fig:robust_results}. We see that adjusting $\beta$ can dramatically improve sample efficiency of PPO and we observe that the improvement in sample efficiency increases with the value of used $\beta$. While standard PPO is unable to learn on Humanoid task Approximate IS Policy Optimization can provide meaningful policy updates leading to much better performance.

\textbf{Different schemes of $\alpha_t$}
We also experiment with choices of $\alpha_t = (0,0,\ldots, 0, \beta_3, \beta_2, \beta_1)$ and report the results in Table \ref{tab:best_alpha_results}. We see that introducing further non-zero components to $\alpha_t$ provides improvements when first and second component are added but plateaus for third added component. 
Investigated choices of $\alpha_t$ again provide improvements on five environments out of eight tasks. 
We also observe that using more complicated schemes of $\alpha_t$ than two non-zero components can provide better performance as visible on HumanoidStandup task but generally lead to similar results while requiring additional computation and slightly slowing down the algorithm. 
We provide extensive comparison of different choices of $\alpha_t$ in the Appendix. 

\section{Conclusions}
We introduced a novel approach to policy optimization based on the approximation to importance sampling. We demonstrated that the introduced proxy objective allows to trade-off bias and variance and covers surrogate objective from \citet{pmlr-v37-schulman15} and importance sampling as special cases. We also derived theoretical results analyzing bias and variance of the introduced objective followed by providing results unifying the previous work on policy improvement. The derived policy optimization procedure led to increase in sample efficiency up to $48 \%$ on high dimensional environments such as Humanoid.


\bibliography{main}
\bibliographystyle{icml2020}

\newpage
\onecolumn
\appendix
\section{Appendix}

\subsection{Skipped proofs}

We begin by providing \reflemma{lem:truncated_bound} which we will use to derive proofs of claims present in text.
\begin{lemma}
\label{lem:truncated_bound}
For random variables $\{X_t\}_{t=0}^\infty$ such that $\exists C > 0 \ \forall t \ge 0  \ |X_t| \le C$ have that
\begin{align}
    \big (\sum_{t \ge 0}^T \gamma^t
    X_t \big )^2
    \le \frac{1-\gamma^{T+1}}{1-\gamma} \sum_{t \ge 0}^T \gamma^t X_t^2. 
\end{align}
\end{lemma}
\begin{proof}
By Cauchy-Schwarz inequality we have that:
\begin{align}
    \Big (\sum_{t \ge 0}^T \gamma^t
    X_t \Big )^2 = 
    \Big (\sum_{t \ge 0}^T \gamma^{t/2}
    \gamma^{t/2} X_t \Big )^2 \le
    \sum_{t \ge 0}^T \gamma^{t} \sum_{t \ge 0}^T \gamma^{t} X_t^2.
\end{align}
By summing geometric series we obtain 
\begin{align}
    \Big (\sum_{t \ge 0}^T \gamma^t
    X_t \Big )^2 \le
    \frac{1-\gamma^{T+1}}{1-\gamma} \sum_{t \ge 0}^T \gamma^{t} X_t^2.
\end{align}
\end{proof}

\begin{proof}[Derivation of \reflemma{lem:L_var}]
By applying \reflemma{lem:truncated_bound} we have that
\begin{align}
    \Big (\sum_{t \ge 0}^T \gamma^t
    \rhoa A^\pi (s_t, a_t) \Big )^2
    \le \frac{1-\gamma^{T+1}}{1-\gamma} \sum_{t \ge 0}^T \gamma^t (\rhoa)^2 A^\pi (s_t, a_t)^2 . 
\end{align}
We also have that
\begin{align}
    & \sum_{t \ge 0}^T \gamma^t (\rhoa)^2 A^\pi (s_t, a_t)^2  =   \sum_{t \ge 0}^T \gamma^t \prod_{i=0}^T (\rho_i^{\alpha_t^i})^2 A^\pi (s_t, a_t)^2   \\
    &\le  \epsilon^2 \sum_{t \ge 0}^T \gamma^t \prod_{i=0}^T (\rho_i^{\alpha_t^i})^2
  \le \epsilon^2 \sum_{t \ge 0}^T \gamma^t \prod_{i=0}^T C_{\rho}^{2 \alpha_t^i}  \\
  &    \le \epsilon^2 \sum_{t \ge 0}^T \gamma^t  C_{\rho}^{2 \sum_{i=0}^T \alpha_t^i} = 
  \epsilon^2 \sum_{t \ge 0}^T \gamma^t  C_{\rho}^{2 \norm {\alpha_t}_1},
\end{align}
where we have used $\epsilon = \max_{s \in \mathcal{S}, a \in \mathcal{A}} A^\pi(s,a)$ and $\max_{s \in \mathcal{S}, a \in \mathcal{A}} \rho(s,a) < C_\rho$. We now have 
\begin{align}
    &\E_{\tau \sim \pi} \Big (\sum_{t \ge 0}^T \gamma^t
    \rhoa A^\pi (s_t, a_t) \Big )^2
    \le \frac{1-\gamma^{T+1}}{1-\gamma} \E_{\tau \sim \pi} \sum_{t \ge 0}^T \gamma^t (\rhoa)^2 A^\pi (s_t, a_t)^2\\ &\le  \epsilon^2 \E_{\tau \sim \pi}  \sum_{t \ge 0}^T \gamma^t  C_{\rho}^{2 \norm {\alpha_t}_1} =
    \epsilon^2   \sum_{t \ge 0}^T \gamma^t  C_{\rho}^{2 \norm {\alpha_t}_1}, 
\end{align}
and because  $\mathbb{V} X = \E X^2 - [\E X]^2 \le \E X^2$ we also have
\begin{align}
     \V_{\tau \sim \pi} \Big (\sum_{t \ge 0}^T \gamma^t
    \rhoa A^\pi (s_t, a_t) \Big )^
    \le \frac{1-\gamma^{T+1}}{1-\gamma} \epsilon^2  \sum_{t \ge 0}^T \gamma^t  C_{\rho}^{2 \norm {\alpha_t}_1}.
\end{align}
\end{proof}

\begin{proof}[Derivation of \reflemma{lem:L_bias}]
We have that 
\begin{align}
    & \sum_{t \ge 0}^T \gamma^t \rhoa A^\pi (s_t, a_t) -  \sum_{t \ge 0}^T \gamma^t \rho_{0:t} A^\pi (s_t, a_t)    =
          \sum_{t \ge 0}^T \gamma^t \big (\rhoa  - \rho_{0:t} \big ) A^\pi (s_t, a_t)   =   \sum_{t \ge 0}^T  \gamma^t \rho_{0:t} \big (\rhoa  - 1 \big ) A^\pi (s_t, a_t)   .
\end{align}
By taking expectation $\E_{\tau \sim \pi}$ on both sides we obtain
\begin{align}
    \E_{\tau \sim \pi} \sum_{t \ge 0}^T \gamma^t \rhoa A^\pi (s_t, a_t) -  \sum_{t \ge 0}^T \gamma^t \rho_{0:t} A^\pi (s_t, a_t) = \E_{\tau \sim \pi} \sum_{t \ge 0}^T  \gamma^t \rho_{0:t} \big (\rhoa  - 1 \big ) A^\pi (s_t, a_t)   .
\end{align}
Next we apply absolute value to both sides which results in 
\begin{align}
      &\big | \E_{\tau \sim \pi} \sum_{t \ge 0}^T \gamma^t \rhoa A^\pi (s_t, a_t) -  \sum_{t \ge 0}^T \gamma^t \rho_{0:t} A^\pi (s_t, a_t) \big | = \big |\E_{\tau \sim \pi} \sum_{t \ge 0}^T  \gamma^t \rho_{0:t} \big (\rhoa  - 1 \big ) A^\pi (s_t, a_t)  \big | \\
      & =\big | \E_{\tau \sim \tilde \pi} \sum_{t \ge 0}^T  \gamma^t \big (\rhoa  - 1 \big ) A^\pi (s_t, a_t)\big | \le \E_{\tau \sim \tilde \pi} \big |\sum_{t \ge 0}^T  \gamma^t \big (\rhoa  - 1 \big ) A^\pi (s_t, a_t) \big |\\
      & \le \E_{\tau \sim \tilde \pi} \sum_{t \ge 0}^T  \gamma^t \big | \rhoa  - 1  \big | \big | A^\pi (s_t, a_t) \big | \le \E_{\tau \sim \tilde \pi} \epsilon \sum_{t \ge 0}^T  \gamma^t \big | \rhoa  - 1  \big |.
\end{align}
where we used $\epsilon = \max_{s \in \mathcal{S}, a \in \mathcal{A}} A^\pi(s,a)$.
\end{proof}

\begin{proof}[ Derivation of \reflemma{lem:nablaL_var}]
Recall that based on \refeqn{eqn:grad} we have that
\begin{equation}
         \nabla \rhoa (s_{0:t} , a_{0:t})  A^\pi (s_t, a_t) = \rhoa \sum_{i=0}^t \alpha_t^i \frac{\pi(a_i|s_i)}{\tilde\pi(a_i|s_i)} \nabla \tilde\pi(a_i|s_i)  A^\pi (s_t, a_t).
\end{equation}
We will use the fact that $\mathrm{tr} \C [X,X] = \sum_{i=1}^N  \E X_i^2 - [\E X_i]^2 \le \sum_{i=1}^N  \E X_i^2  = \E \sum_{i=1}^N  X_i^2  = \E \norm{X}^2_2$ to obtain
\begin{align}
    & \mathrm{tr} \C_{\tau \sim \pi} \Big [\nabla \sum_{t\ge 0} \gamma^t \rhoa (s_{0:t} , a_{0:t})  A^\pi (s_t, a_t) ,\nabla \sum_{t\ge 0} \gamma^t \rhoa (s_{0:t} , a_{0:t})  A^\pi (s_t, a_t)   \Big ] \label{eqn:bound_on_tr}\\
    & \le    \E_{\tau \sim \pi} \norm{ \sum_{t\ge 0} \gamma^t \nabla \rhoa (s_{0:t} , a_{0:t})  A^\pi (s_t, a_t) } ^2_2 \\
    &=    \E_{\tau \sim \pi} \norm{\sum_{t\ge 0} \gamma^t \rhoa \sum_{i=0}^t \alpha_t^i \frac{\pi(a_i|s_i)}{\tilde\pi(a_i|s_i)} \nabla \tilde\pi(a_i|s_i)  A^\pi (s_t, a_t) }  ^2_2.
\end{align}
Therefore to obtain an upper bound on $\mathrm{tr} \C_{\tau \sim \pi} \Big [\nabla \sum_{t\ge 0} \gamma^t \rhoa (s_{0:t} , a_{0:t})  A^\pi (s_t, a_t) ,\nabla \sum_{t\ge 0} \gamma^t \rhoa (s_{0:t} , a_{0:t})  A^\pi (s_t, a_t)   \Big ]$ 
we proceed by upper bounding term $\norm{ \sum_{t\ge 0} \gamma^t \rhoa \sum_{i=0}^t \alpha_t^i \frac{\pi(a_i|s_i)}{\tilde\pi(a_i|s_i)} \nabla \tilde\pi(a_i|s_i)  A^\pi (s_t, a_t) }_2$ as follows:
\begin{align}
 &\norm{ \sum_{t\ge 0} \gamma^t \rhoa \sum_{i=0}^t \alpha_t^i \frac{\pi(a_i|s_i)}{\tilde\pi(a_i|s_i)} \nabla \tilde\pi(a_i|s_i)  A^\pi (s_t, a_t) }_2 \\
 & \le  \sum_{t\ge 0} \norm{ \gamma^t \rhoa \sum_{i=0}^t \alpha_t^i \frac{\pi(a_i|s_i)}{\tilde\pi(a_i|s_i)} \nabla \tilde\pi(a_i|s_i)  A^\pi (s_t, a_t) }_2 \\
  &  \le \sum_{t\ge 0} \gamma^t \rhoa \sum_{i=0}^t \norm{ \alpha_t^i \frac{\pi(a_i|s_i)}{\tilde\pi(a_i|s_i)} \nabla \tilde\pi(a_i|s_i)  A^\pi (s_t, a_t) }_2 \\
    &  \le \sum_{t\ge 0} \gamma^t \rhoa \sum_{i=0}^t  \alpha_t^i |A^\pi (s_t, a_t) | \Big |  \frac{\pi(a_i|s_i)}{\tilde\pi(a_i|s_i)} \Big | \norm{\nabla \tilde\pi(a_i|s_i)  }_2.
\end{align}
We further have that $\norm{\nabla \tilde\pi(a_i|s_i)  }_2 \le C_\partial^2$, $|A^\pi (s_t, a_t) | \le \epsilon$ and $\big |\frac{\tilde\pi(a_i|s_i)}{\pi(a_i|s_i)} \big | < C_\delta$ and $\norm{\alpha_t}_1 = \sum_{i=0}^t \alpha_t^i  $. It follows that 
\begin{align}
    & \sum_{t\ge 0} \gamma^t \rhoa \sum_{i=0}^t  \alpha_t^i |A^\pi (s_t, a_t) | \Big |  \frac{\pi(a_i|s_i)}{\tilde\pi(a_i|s_i)}  \Big | \norm{\nabla \tilde\pi(a_i|s_i)  }  ^2_2  \\
    & \le \epsilon C_\partial C_\delta \sum_{t\ge 0} \gamma^t \rhoa \sum_{i=0}^t \alpha_t^i  = \epsilon C_\partial  C_\delta \sum_{t\ge 0} \gamma^t \rhoa \norm{\alpha_t}_1.
\end{align}

We apply Cauchy-Schwarz inequality to obtain
\begin{align}
    \sum_{t\ge 0} \gamma^t \rhoa \norm{\alpha_t}_1 \le \sqrt{\sum_{t\ge 0} \gamma^t (\rhoa)^2}  \sqrt {\sum_{t\ge 0} \gamma^t \norm{\alpha_t}^2_1}.
\end{align}
So we have
\begin{align}
&\norm{ \sum_{t\ge 0} \gamma^t \rhoa \sum_{i=0}^t \alpha_t^i \frac{\pi(a_i|s_i)}{\tilde\pi(a_i|s_i)} \nabla \tilde\pi(a_i|s_i)  A^\pi (s_t, a_t) }_2  \le \epsilon C_\partial  C_\delta  \sqrt{\sum_{t\ge 0} \gamma^t (\rhoa)^2}  \sqrt {\sum_{t\ge 0} \gamma^t \norm{\alpha_t}^2_1}.
\end{align}
It follows that
\begin{align}
&\norm{ \sum_{t\ge 0} \gamma^t \rhoa \sum_{i=0}^t \alpha_t^i \frac{\pi(a_i|s_i)}{\tilde\pi(a_i|s_i)} \nabla \tilde\pi(a_i|s_i)  A^\pi (s_t, a_t) }_2^2  \le \epsilon^2 C_\partial^2  C_\delta^2  \sum_{t\ge 0} \gamma^t (\rhoa)^2  \sum_{t\ge 0} \gamma^t \norm{\alpha_t}^2_1.
\end{align}
So we have that
\begin{align}
    \E_{\tau \sim \pi} \norm{ \sum_{t\ge 0} \gamma^t \rhoa \sum_{i=0}^t \alpha_t^i \frac{\pi(a_i|s_i)}{\tilde\pi(a_i|s_i)} \nabla \tilde\pi(a_i|s_i)  A^\pi (s_t, a_t) }_2^2  \le  \epsilon^2 C_\partial^2  C_\delta^2   \sum_{t\ge 0} \gamma^t \norm{\alpha_t}^2_1  \E_{\tau \sim \pi} \sum_{t\ge 0} \gamma^t (\rhoa)^2 ,
\end{align}
We will now use the assumption hat $\E_{\tau \sim \pi} \sum_{t\ge 0} \gamma^t (\rhoa)^2 < C^\gamma$ so we have that $\sum_{t\ge 0} \gamma^t (\rhoa)^2< C^{\gamma}$. This yields
\begin{align}
    \E_{\tau \sim \pi} \norm{ \sum_{t\ge 0} \gamma^t \rhoa \sum_{i=0}^t \alpha_t^i \frac{\pi(a_i|s_i)}{\tilde\pi(a_i|s_i)} \nabla \tilde\pi(a_i|s_i)  A^\pi (s_t, a_t) }_2^2  \le  \epsilon^2 C_\partial^2  C_\delta^2 C^{\gamma} \sum_{t\ge 0} \gamma^t \norm{\alpha_t}^2_1 .
\end{align}

By comparing to inequalities obtained starting from  \ref{eqn:bound_on_tr} it follows that
\begin{align}
  \mathrm{tr} \C_{\tau \sim \pi} \Big [\nabla \sum_{t\ge 0} \gamma^t \rhoa (s_{0:t} , a_{0:t})  A^\pi (s_t, a_t) ,\nabla \sum_{t\ge 0} \gamma^t \rhoa (s_{0:t} , a_{0:t})  A^\pi (s_t, a_t)   \Big ] \le  \epsilon^2 C_\partial^2  C_\delta^2 C^{\gamma} \sum_{t\ge 0} \gamma^t \norm{\alpha_t}^2_1 ,
\end{align}
which concludes the proof.
\end{proof}

\begin{proof}[Derivation of \reflemma{lem:nablaL_bias}]
We first find an upper bound on $\norm{\nabla \sum_{t \ge 0} \rhoa A^\pi(s_t, a_t) - \nabla \sum_{t \ge 0} \rho_{0:t} A^\pi(s_t, a_t)}_2 $. We have that
\begin{align}
    &\norm{\nabla \sum_{t \ge 0} \rhoa A^\pi(s_t, a_t) - \nabla \sum_{t \ge 0} \rho_{0:t} A^\pi(s_t, a_t)}_2 \\
    &= \norm{\sum_{t \ge 0} \nabla \rhoa A^\pi(s_t, a_t) -  \nabla \rho_{0:t} A^\pi(s_t, a_t) \Big ]}_2 \\
    & \le \sum_{t \ge 0} \norm{\nabla \rhoa A^\pi(s_t, a_t) -  \nabla \rho_{0:t} A^\pi(s_t, a_t) }_2.
\end{align}
By \refeqn{eqn:grad} we can express $\nabla \rhoa (s_{0:t} , a_{0:t})  A^\pi (s_t, a_t) -  \nabla \rho_{0:t} (s_{0:t} , a_{0:t})  A^\pi (s_t, a_t)$ as
\begin{align}
    &\nabla \rhoa (s_{0:t} , a_{0:t})  A^\pi (s_t, a_t) -
     \nabla \rho_{0:t} (s_{0:t} , a_{0:t})  A^\pi (s_t, a_t)\\
     &= \rhoa \sum_{i=0}^t \alpha_t^i \frac{\pi(a_i|s_i)}{\tilde\pi(a_i|s_i)} \nabla \tilde\pi(a_i|s_i)  A^\pi (s_t, a_t) - \rho_{0:t} \sum_{i=0}^t \frac{\pi(a_i|s_i)}{\tilde\pi(a_i|s_i)} \nabla \tilde\pi(a_i|s_i)  A^\pi (s_t, a_t) \\
      &= \sum_{i=0}^t \alpha_t^i \rhoa  \frac{\pi(a_i|s_i)}{\tilde\pi(a_i|s_i)} \nabla \tilde\pi(a_i|s_i)  A^\pi (s_t, a_t) - \sum_{i=0}^t  \rho_{0:t} \frac{\pi(a_i|s_i)}{\tilde\pi(a_i|s_i)} \nabla \tilde\pi(a_i|s_i)  A^\pi (s_t, a_t) \\
      &= \sum_{i=0}^t \big (\alpha_t^i \rhoa   -  \rho_{0:t} \big ) \frac{\pi(a_i|s_i)}{\tilde\pi(a_i|s_i)} \nabla \tilde\pi(a_i|s_i)  A^\pi (s_t, a_t).
\end{align}
So it follows that
\begin{align}
    & \norm{\nabla \rhoa (s_{0:t} , a_{0:t})  A^\pi (s_t, a_t) -\nabla \rho_{0:t} (s_{0:t} , a_{0:t})  A^\pi (s_t, a_t)}_2 \\
    & \norm{ \sum_{i=0}^t \big (\alpha_t^i \rhoa   -  \rho_{0:t} \big ) \frac{\pi(a_i|s_i)}{\tilde\pi(a_i|s_i)} \nabla \tilde\pi(a_i|s_i)  A^\pi (s_t, a_t)}_2 \\
    & \le \sum_{i=0}^t \norm{ \big (\alpha_t^i \rhoa   -  \rho_{0:t} \big ) \frac{\pi(a_i|s_i)}{\tilde\pi(a_i|s_i)} \nabla \tilde\pi(a_i|s_i)  A^\pi (s_t, a_t)}_2 \\
    &  = \sum_{i=0}^t  \big | \alpha_t^i \rhoa   -  \rho_{0:t} \big |  \big | \frac{\pi(a_i|s_i)}{\tilde\pi(a_i|s_i)}\big | \big |   A^\pi (s_t, a_t) \big |
    \norm{\nabla \tilde\pi(a_i|s_i) }_2. 
\end{align}
Recall that we have that $\epsilon = \max_{s \in \mathcal S,a \in \mathcal A}|A^\pi(s,a) |$ and $\norm{\nabla \tilde\pi(a_i|s_i) }_2 \le C^\partial$. We proceed by further bounding the term $\sum_{i=0}^t  \big | \alpha_t^i \rhoa   -  \rho_{0:t} \big |  \big | \frac{\pi(a_i|s_i)}{\tilde\pi(a_i|s_i)}\big | \big |   A^\pi (s_t, a_t) \big |\norm{\nabla \tilde\pi(a_i|s_i) }_2 $ in the following way:
\begin{align}
&\sum_{i=0}^t  \big | \alpha_t^i \rhoa   -  \rho_{0:t} \big |  \big | \frac{\pi(a_i|s_i)}{\tilde\pi(a_i|s_i)}\big | \big |   A^\pi (s_t, a_t) \big |
    \norm{\nabla \tilde\pi(a_i|s_i) }_2 \\
      &  \le \epsilon C^\partial \sum_{i=0}^t  \big | \alpha_t^i \rhoa   -  \rho_{0:t} \big | \le \epsilon C^\partial  \sum_{i=0}^t  \big | \alpha_t^i \rhoa - \rhoa + \rhoa  - \rho_{0:t} \big | \\ 
  &  \le \epsilon C^\partial  \sum_{i=0}^t  \big | (\alpha_t^i -1)\rhoa| + |\rhoa  - \rho_{0:t} \big |  = \epsilon C^\partial  \sum_{i=0}^t  \big | (\alpha_t^i -1)\rhoa| + \sum_{i=0}^t |\rhoa  - \rho_{0:t} \big | \\
      & = \epsilon C^\partial  \sum_{i=0}^t  \big | (\alpha_t^i -1)\rhoa| + (t+1) |\rhoa  - \rho_{0:t} \big | = \epsilon C^\partial  \rhoa \sum_{i=0}^t  \big | \alpha_t^i -1| + \sum_{i=0}^t |\rhoa  - \rho_{0:t} \big | \\
      & = \epsilon C^\partial  \sum_{i=0}^t  \big | (\alpha_t^i -1)\rhoa| + (t+1) |\rhoa  - \rho_{0:t} \big |  = \epsilon C^\partial  \rhoa \sum_{i=0}^t  \big | \alpha_t^i -1| +  (t+1) |\rhoa  - \rho_{0:t} \big | \\
  & = \epsilon C^\partial  \rhoa \norm{ \alpha_t -\mathbf{1}}_1 +  (t+1) |\rhoa  - \rho_{0:t} \big |.
\end{align}
As a result we have that
\begin{align}
    &\norm{\nabla \sum_{t \ge 0} \gamma^t \rhoa A^\pi(s_t, a_t) - \nabla \sum_{t \ge 0} \gamma^t \rho_{0:t} A^\pi(s_t, a_t)}_2 \le  \epsilon C^\partial  \sum_{t \ge 0}  \gamma^t \Big [ \rhoa \norm{ \alpha_t -\mathbf{1}}_1 +  (t+1) |\rhoa  - \rho_{0:t} \big | \Big ].
\end{align}
\end{proof}

\begin{proof}[Derivation of \reflemma{lem:mpie}]
The equality in \refeqn{eqn:lem_kakade} is symmetrical w.r.t. policies $\pi$ and $\tilde\pi$. We can write
\begin{equation}
\eta(\pi) - \eta(\tilde\pi) = \frac{1}{1 - \gamma} \mathbb{E}_{s \sim d^{\pi}, a \sim \pi (\cdot|s)}  A^{\tilde\pi} (s,a).
\end{equation}
It follows that
\begin{align}
&\eta(\pi) - \eta(\tilde\pi) = \frac{1}{1 - \gamma} \mathbb{E}_{s \sim d^{\pi}, a \sim \pi(\cdot|s)}  A^{\tilde\pi} (s,a) =   \\
& = \frac{1}{1 - \gamma} \mathbb{E}_{s \sim d^{\pi}} \Big[\int \tilde \pi (a|s) A^\pi (s,a) da -\int  \tilde \pi (a|s) A^\pi (s,a) da + \int  \pi (a|s) A^{\tilde\pi}(s,a) da \Big]  \\
& = -L_\pi(\tilde\pi) + \frac{1}{1 - \gamma} \mathbb{E}_{s \sim d^{\pi}} \Big[ \int  \tilde \pi (a|s) A^\pi (s,a) da + \int  \pi (a|s) A^{\tilde\pi}(s,a) da \Big]  \\
&  = -L_\pi(\tilde\pi) + \frac{1}{1 - \gamma} \mathbb{E}_{s \sim d^{\pi}} \Big[\int  \tilde \pi (a|s) Q^\pi (s,a) da + \int  \pi (a|s) Q^{\tilde\pi}(s,a) da - V^{\tilde\pi}(s) - V^{\pi}(s) \Big]  \\
&  = - L_\pi(\tilde\pi) + \frac{1}{1 - \gamma} \mathbb{E}_{s \sim d^{\pi}} \Big[\int  \tilde \pi (a|s) Q^\pi (s,a) da + \int  \pi (a|s) Q^{\tilde\pi}(s,a) da \nonumber \\ 
&  - \int \tilde\pi(a|s) Q^{\tilde\pi}(s, a) da - \int \pi(a|s) Q^{\pi}(s, a) da \Big]  \\
&  = - L_\pi(\tilde\pi) + \frac{1}{1 - \gamma} \mathbb{E}_{s \sim d^{\pi}} \Big[ \int \tilde\pi(a|s) \big(Q^{\pi}(s, a) - Q^{\tilde\pi}(s, a)\big) da  + \int \pi(a|s) \big(Q^{\tilde\pi}(s, a) - Q^{\pi}(s, a)\big) da  \Big]  \\
&  = -L_\pi(\tilde\pi) + \frac{1}{1 - \gamma} \mathbb{E}_{s \sim d^{\pi}} \Big[ \int \tilde\pi(a|s) \big(Q^{\pi}(s, a) - Q^{\tilde\pi}(s, a)\big) da  - \int \pi(a|s) \big(Q^{\pi}(s, a) - Q^{\tilde\pi}(s, a)\big) da  \Big]  \\
&  = - L_\pi(\tilde\pi) + \frac{1}{1 - \gamma} \mathbb{E}_{s \sim d^{\pi}} \Big[ \int \big(\tilde\pi(a|s) - \pi(a|s)\big)\big(Q^{\pi}(s, a) - Q^{\tilde\pi}(s, a) \big)da \Big] \\
&  = - L_\pi(\tilde\pi) - \frac{1}{1 - \gamma} \mathbb{E}_{s \sim d^{\pi}} \Big[ \int \big(\tilde\pi(a|s) - \pi(a|s)\big)\big(Q^{\tilde\pi}(s, a) - Q^{\pi}(s, a) \big) da \Big].
\end{align}
So we obtain that
\begin{align}
\eta(\pi) - \eta(\tilde\pi) = -L_\pi(\tilde\pi) -  \frac{1}{1 - \gamma} \mathbb{E}_{s \sim d^{\pi}}  \int \big(\tilde\pi(a|s) - \pi(a|s)\big)\big(Q^{\tilde\pi}(s, a) - Q^{\pi}(s, a) \big) da.
\end{align}
After multiplying by $-1$ this gives
\begin{align}
\eta(\tilde\pi) - \eta(\pi)= L_\pi(\tilde\pi) + \frac{1}{1 - \gamma} \mathbb{E}_{s \sim d^{\pi}}  \int \big(\tilde\pi(a|s) - \pi(a|s)\big)\big(Q^{\tilde\pi}(s, a) - Q^{\pi}(s, a) \big) da.
\end{align}
\end{proof}
Note that the integral $\int \big(\tilde\pi(a|s) - \pi(a|s)\big)\big(Q^{\tilde\pi}(s, a) - Q^{\pi}(s, a) \big) da$
can be viewed as a weighted total variation distance between $\pi$ and $\tilde \pi$ where weights are defined by $Q^{\tilde\pi}(s, a) - Q^{\pi}(s, a) $. It follows that the divergence of policies $\pi$ and $\tilde\pi$ on state-action pairs $(s,a)$ having large difference in value $Q^{\tilde\pi}(s, a) - Q^{\pi}(s, a)$ is causing degradation in the performance of approximation of $L_\pi(\tilde\pi)$. Also $Q^{\tilde\pi}(s, a) - Q^{\pi}(s, a)$ can be viewed as weighting current difference between policies $\pi$ and $\tilde\pi$ by delayed effect of the difference arising from $Q^{\tilde\pi}(s, a) - Q^{\pi}(s, a)$.

\subsection{Derivations of the results from literature}

We first provide the following Corollary which we will use to derive Policy Gradient Theorems. The result follows from the following algebraic transformations applied to \refeqn{eqn:mpie}.

\begin{corollary}[Value dependency equality]
\label{cor:vde}
Given two policies $\tilde\pi$ and $\pi$,
\begin{equation}
\label{eqn:vde}
\eta(\tilde\pi) - \eta(\pi) = \frac{1}{1-\gamma} \E_{s \sim d^\pi} \int \big( \tilde\pi(a|s) - \pi(a|s) \big) Q^{\tilde\pi}(s,a) da.
\end{equation}
\end{corollary}

\begin{proof}
\begin{align}
& \eta(\tilde\pi) = \eta(\pi) + L_\pi(\tilde \pi) + \frac{1}{1-\gamma} \E_{s \sim d^\pi}  \int \big (Q^{\tilde\pi}(s,a) - Q^{\pi}(s,a) \big)\big( \tilde\pi(a|s) - \pi(a|s) \big) da  \\
& = \eta(\pi) + \frac{1}{1-\gamma} \E_{s \sim d^\pi} \int \tilde\pi (a|s) A^\pi (s,a) da + \frac{1}{1-\gamma} \E_{s \sim d^\pi} \int \big(Q^{\tilde\pi}(s,a) - Q^{\pi}(s,a) \big) \big ( \tilde\pi(a|s) - \pi(a|s) \big) da \\
& = \eta(\pi) + \frac{1}{1-\gamma} \E_{s \sim d^\pi} \Big [  \int\tilde\pi (a|s) Q^\pi (s,a) da  - V^\pi(s) \Big ]
 + \frac{1}{1-\gamma} \E_{s \sim d^\pi}  \int \big(Q^{\tilde\pi}(s,a) - Q^{\pi}(s,a) \big) \big( \tilde\pi(a|s) - \pi(a|s) \big) da\\
& = \eta(\pi) + \frac{1}{1-\gamma} \E_{s \sim d^\pi} \Big [   \int\tilde\pi (a|s) Q^\pi (s,a) da  - \int \pi (a|s) Q^\pi (s,a) da \Big ] + \nonumber \\
& + \frac{1}{1-\gamma} \E_{s \sim d^\pi}  \int \Big [ \tilde\pi(a|s) Q^{\tilde\pi}(s,a) - \tilde\pi(a|s) Q^{\pi}(s,a)  - \pi(a|s) Q^{\tilde\pi}(s,a) + \pi(a|s) Q^{\pi}(s,a) \Big ] da \\
& = \eta(\pi) + \frac{1}{1-\gamma} \E_{s \sim d^\pi}  \int  \tilde\pi(a|s) Q^{\tilde\pi}(s,a)  - \pi(a|s) Q^{\tilde\pi}(s,a) da \\
& = \eta(\pi) + \frac{1}{1-\gamma} \E_{s \sim d^\pi}  \int \big(\tilde\pi(a|s) - \pi(a|s)\big) Q^{\tilde\pi}(s,a) da.
\end{align}
\end{proof}

\begin{proof}[Alternative derivation of \refcor{cor:vde}]
Again, by using symmetry in \refeqn{eqn:lem_kakade} we obtain:
\begin{equation}
    \eta(\pi) - \eta(\tilde\pi) =  + \frac{1}{1-\gamma} \mathbb{E}_{s \sim d^\pi} \int \pi(a|s) A^{\tilde\pi}(s,a) da.
\end{equation}
It follows that
\begin{align}
    &\eta(\tilde\pi) = \eta(\pi) - \frac{1}{1-\gamma} \mathbb{E}_{s \sim d^\pi} \int \pi(a|s) A^{\tilde\pi}(s,a) da\\ 
    &\eta(\tilde\pi) = \eta(\pi) - \frac{1}{1-\gamma} \mathbb{E}_{s \sim d^\pi} \Big [ -\int \tilde \pi(a|s) A^{\tilde\pi}(s,a) da  + \int \pi(a|s) A^{\tilde\pi}(s,a) da \Big ]\\ 
    &\eta(\tilde\pi) = \eta(\pi) - \frac{1}{1-\gamma} \mathbb{E}_{s \sim d^\pi} \Big [ \int \tilde \pi(a|s) \Big ( V^{\tilde\pi}(s) - Q^{\tilde\pi}(s,a)  \Big ) da + \int \pi(a|s) \Big ( Q^{\tilde\pi}(s,a)- V^{\tilde\pi}(s) \Big ) da  \Big ]\\ 
  &\eta(\tilde\pi) = \eta(\pi) - \frac{1}{1-\gamma} \mathbb{E}_{s \sim d^\pi} \Big [ -\int \tilde \pi(a|s) Q^{\tilde\pi}(s,a) da   + V^{\tilde\pi}(s)  + \int \pi(a|s)  Q^{\tilde\pi}(s,a)da - V^{\tilde\pi}(s)   \Big ]\\ 
    &\eta(\tilde\pi) = \eta(\pi) - \frac{1}{1-\gamma} \mathbb{E}_{s \sim d^\pi} \Big [ -\int \tilde \pi(a|s) Q^{\tilde\pi}(s,a) da + \int \pi(a|s) Q^{\tilde\pi}(s,a) da \Big ]\\ 
    &\eta(\tilde\pi) = \eta(\pi) + \frac{1}{1-\gamma} \mathbb{E}_{s \sim d^\pi}  \int \big (\tilde \pi(a|s)  - \pi(a|s) \big) Q^{\tilde\pi}(s,a) da.
\end{align}
\end{proof}

We begin by providing short derivations of theorems obtained in previous works of \citet{Sutton:1999:PGM:3009657.3009806,pmlr-v32-silver14, ciosek2018expected}. To simplify notation, for parametrized policies, we denote target policy $\tilde \pi$ as $\pi_\theta$ and behavior policy $\pi$ as $\pi_{\theta_0}$. 

\refeqn{eqn:vde} requires access to state-action value function $Q^{\pi_\theta}$ of the target policy. Nevertheless, we show that \refcor{cor:vde} can be used to unify and easily derive previously proven policy gradient theorems \citep{Sutton:1999:PGM:3009657.3009806, pmlr-v32-silver14, ciosek2018expected}. These results follow from the fact that the state occupancy measure on the RHS does not depend on target policy $\pi_\theta$, so the RHS can be easily differentiated with respect to its parameters.

\begin{proof}[Derivation of Policy Gradient Theorem (Theorem 1) from \cite{Sutton:1999:PGM:3009657.3009806}]
We differentiate expression for $\eta(\pi_\theta)$ from Corollary \ref{cor:vde}:
\begin{align}
&\frac{\partial}{\partial\theta}\eta(\pi_\theta) = \frac{\partial}{\partial\theta} \Bigg[ \eta(\pi_{\theta_0}) + \frac{1}{1-\gamma} \E_{s \sim d^\pi} \int Q^{\pi_{\theta}}(s,a) \big( \pi_{\theta}(a|s) - \pi_{\theta_0}(a|s) \big)da  \Bigg ] = \\
& = \frac{1}{1-\gamma} \Bigg [ \E_{s \sim d^\pi} \int \frac{\partial}{\partial\theta}  Q^{\pi_\theta}(s,a) \big( \pi_{\theta}(a|s) - \pi_{\theta_0}(a|s) \big) da +  
 \int  Q^{\pi_\theta}(s,a) \frac{\partial}{\partial\theta} \big ( \pi_\theta(a|s) - \pi_{\theta_0}(a|s) \big) da \Bigg ] \\
& =\frac{1}{1-\gamma} \Bigg [ \E_{s \sim d^\pi} \int \frac{\partial}{\partial\theta}  Q^{\pi_\theta}(s,a) \big( \pi_{\theta}(a|s) - \pi_{\theta_0}(a|s) \big) da +  
 \int  Q^{\pi_\theta}(s,a) \frac{\partial}{\partial\theta}\pi_\theta(a|s) da \Bigg ].
\end{align}
By evaluating derivative at $\theta = \theta_0$ we obtain:
\begin{equation}
\frac{\partial}{\partial\theta} \eta(\pi_\theta)\Big|_{\theta= \theta_0} =  \frac{1}{1-\gamma} \E_{s \sim d^\pi} \int \frac{\partial}{\partial\theta}\pi_{\theta}(a|s)\Big|_{\theta= \theta_0}  Q^{\pi_{\theta_0}}(s,a) da.
\end{equation}
After applying $\frac{\partial}{\partial x} f(x) = f(x) \frac{\partial}{\partial x} \log f(x)$:
\begin{equation}
\frac{\partial}{\partial\theta} \eta(\pi_\theta)\Big|_{\theta= \theta_0} =  \frac{1}{1-\gamma}\E_{s \sim d^\pi} \int \pi_{\theta_0}(a|s) \frac{\partial}{\partial\theta}\log \pi_{\theta}(a|s)\Big|_{\theta= \theta_0}  Q^{\pi_{\theta_0}}(s,a) da.
\end{equation}
\end{proof}

\begin{proof}[Derivation of Deterministic Policy Gradient Theorem (Theorem 1 from \citet{pmlr-v32-silver14})]
Again, we differentiate the expression for $\eta(\pi_\theta)$ from Corollary \ref{cor:vde}
\begin{align}
& \frac{\partial}{\partial\theta}\eta(\pi_\theta) = \frac{\partial}{\partial\theta} \Bigg[ \eta(\pi_{\theta_0}) + \frac{1}{1-\gamma} \E_{s \sim d^\pi} Q^{\pi_{\theta}}(s,\pi_{\theta}(s))  - Q^{\pi_{\theta}}(s,\pi_{\theta_0}(s))\Bigg ] = \\
& = \frac{1}{1-\gamma}  \Bigg[\E_{s \sim d^\pi}  \frac{\partial}{\partial\theta} Q^{\pi_{\theta}}(s,\pi_{\theta}(s))  -  \frac{\partial}{\partial\theta} Q^{\pi_{\theta}}(s,\pi_{\theta_0}(s)) \Bigg] \\
& = \frac{1}{1-\gamma}  \Bigg[\E_{s \sim d^\pi}  \frac{\partial}{\partial a} Q^{\pi_{\theta}}(s,a) \frac{\partial}{\partial \theta} \pi_{\theta}(s) +   \frac{\partial}{\partial\theta} Q^{\pi_{\theta}}(s,\pi_{\theta}(s)) -  \frac{\partial}{\partial\theta} Q^{\pi_{\theta}}(s,\pi_{\theta_0}(s)) \Bigg].
\end{align}
Where we calculate total derivative = $ \frac{\partial}{\partial\theta} Q^{\pi_{\theta}}(s,\pi_{\theta}(s)) = \frac{\partial}{\partial \pi_{\theta}(s)} Q^{\pi_{\theta}}(s,\pi_{\theta}(s)) + \frac{\partial}{\partial\theta} Q^{\pi_{\theta}}(s,\pi_{\theta}(s))$ and then use chain rule to get
$\frac{\partial}{\partial \pi_{\theta}(s)} Q^{\pi_{\theta}}(s,\pi_{\theta}(s)) =  \frac{\partial}{\partial a} Q^{\pi_{\theta}}(s,a) |_{a= \pi_{\theta} (s)} \frac{\partial}{\partial \theta} \pi_{\theta}(s)$. By evaluating derivative at $\theta = \theta_0$ we obtain:
\begin{equation}
\frac{\partial}{\partial\theta} \eta(\pi_\theta)\Big|_{\theta= \theta_0} =  \frac{1}{1-\gamma} \E_{s \sim d^\pi}  \frac{\partial}{\partial a} Q^{\pi_{\theta_0}}(s,a) \Big|_{a= \pi_{\theta_0} (s)} \frac{\partial}{\partial \theta} \pi_{\theta}(s)\Big|_{\theta= \theta_0}.
\end{equation}
\end{proof}

\begin{proof}[Derivation of General Policy Gradient Theorem (Theorem 1 from \citet{ciosek2018expected})]
This result also follows from differentiating expression for $\eta(\pi_\theta)$ from Corollary \ref{cor:vde}
\begin{align}
&\frac{\partial}{\partial\theta}\eta(\pi_\theta) = \frac{\partial}{\partial\theta} \Bigg[ \eta(\pi_{\theta_0}) + \frac{1}{1-\gamma} \E_{s \sim d^\pi} \int Q^{\pi_{\theta}}(s,a) \big ( \pi_{\theta}(a|s) - \pi_{\theta_0}(a|s) \big ) da \Bigg ] \\
& = \frac{\partial}{\partial\theta} \frac{1}{1-\gamma} \E_{s \sim d^\pi} \Bigg[\int Q^{\pi_{\theta}}(s,a) \pi_{\theta}(a|s) da -\int  \pi_{\theta_0}(a|s) Q^{\pi_{\theta}}(s,a) da\Bigg]= \\ 
& = \frac{\partial}{\partial\theta} \frac{1}{1-\gamma} \E_{s \sim d^\pi}\Bigg[ V^{\pi_\theta}(s) -\int  \pi_{\theta_0}(a|s) Q^{\pi_{\theta}}(s,a) da \Bigg]  \\ 
& = \frac{1}{1-\gamma} \E_{s \sim d^\pi} \Bigg[ \frac{\partial}{\partial\theta} V^{\pi_\theta}(s) - \int  \pi_{\theta_0}(a|s) \frac{\partial}{\partial\theta} Q^{\pi_{\theta}}(s,a) da\Bigg].
\end{align}
By evaluating derivative at $\theta = \theta_0$ we derive
\begin{equation}
\frac{\partial}{\partial\theta}\eta(\pi_\theta)\Big|_{\theta= \theta_0} = \frac{1}{1-\gamma} \E_{s \sim d^\pi} \Bigg [ \frac{\partial}{\partial\theta} V^{\pi_\theta}(s) \Big|_{\theta= \theta_0} - \int  \pi_{\theta_0}(a|s) \frac{\partial}{\partial\theta} Q^{\pi_{\theta}}(s,a)\Big|_{\theta= \theta_0} da \Bigg ].
\end{equation}
\end{proof}

Next, we derive the previously obtained bounds on the quality of approximation of $L_\pi(\tilde\pi)$. We firstly prove a lemma used in these derivations.
\begin{lemma}
\label{lem:inverse_kakade}
Given two policies $\tilde\pi$ and $\pi$, the following equality holds
\begin{align}
&\E_{\tau \sim \tilde \pi | s, a} \sum_{t \ge 1} \gamma^{t-1} A^\pi(s_t,a_t) = \frac{1}{\gamma} \big ( Q^{\tilde\pi}(s,a) - Q^\pi(s,a) \big ).
\end{align}
\end{lemma}
\begin{proof}
We apply the following algebraic transformations
\begin{align}
& \E_{\tau \sim \tilde \pi} [\sum_{t \ge 1} \gamma^{t-1} A^\pi(s_t,a_t) | s_0 = s, a_0 = a] = \sum_{t \ge 1} \E_{s_t, s_{t+1}}[  \gamma^{t-1} \big (r(s_t, a_t) + \gamma V^{\pi}(s_{t+1}) - V^{\pi}(s_t) \big ) | s_0 = s, a_0 = a] \\
& = \E_{\tau \sim \tilde \pi} [- V^{\pi}(s_1) +  \sum_{t \ge 1} \gamma^{t-1} r(s_t, a_t) | s_0 = s, a_0 = a]  = \E_{p(s' |s, a)} - V^{\pi}(s') +  V^{\tilde\pi}(s') =  \E_{p(s' |s, a)} V^{\tilde\pi}(s') - V^{\pi}(s')  \\
& = \frac{1}{\gamma} \E_{p(s'|s,a)} \gamma V^{\tilde\pi}(s') - \gamma V^{\pi}(s') 
 = \frac{1}{\gamma} \E_{p(s'|s,a)} \gamma V^{\tilde\pi}(s') + r(s,a) - \gamma V^{\pi}(s') - r(s,a)  = \frac{1}{\gamma} \big( Q^{\tilde\pi}(s,a) - Q^\pi(s,a) \big ).
\end{align}
\end{proof}

\begin{proof}[Derivation of Corollary 1 from from \citet{pmlr-v70-achiam17a}]
Note that when $\epsilon = \max_s | \E_{a \sim \tilde \pi(\cdot|s)} A^\pi (s,a) |$ we have that $ \frac{\epsilon}{1-\gamma} \ge \E_{\tau \sim \tilde\pi |s,a} \sum_{t\ge 1}\gamma^{t-1} | \E_{a \sim \tilde \pi (\cdot|s_t)} A^\pi(s_t, a)|$. From \reflemma{lem:inverse_kakade} we have that $ Q^{\tilde\pi}(s,a) - Q^\pi(s,a) = \gamma \E_{\tau \sim \tilde \pi |s, a} \sum_{t \ge 1} \gamma^{t-1} A^\pi(s_t,a_t) $. It follows that
\begin{align}
& \Big |\eta(\tilde\pi)- \eta(\pi) -  L_{\tilde\pi}(\pi) \Big | =  \\
& = \frac{1}{1-\gamma}\Big | \E_{s \sim d^\pi} \int \big (Q^{\tilde\pi}(s,a) - Q^\pi (s,a) \big) \big (\tilde\pi(a|s) - \pi(a|s) \big )\Big | da\\
& = \frac{1}{1-\gamma}\Big |\E_{s \sim d^\pi} \int \gamma\Big[ \E_{\tau \sim \tilde\pi |s,a} \sum_{t\ge 1} \gamma^{t-1} A^\pi(s_t, a_t) \Big] \big (\tilde\pi(a|s) - \pi(a|s) \big )\Big|  da\\
& \le \frac{1}{1-\gamma} \E_{s \sim d^\pi} \int \frac{\gamma\epsilon}{1-\gamma} \Big |\tilde\pi(a|s) - \pi(a|s)\Big| da\\
& = \frac{2\epsilon \gamma}{(1-\gamma)^2} \E_{s \sim d^\pi} \frac{1}{2} \int \Big|\tilde\pi(a|s) - \pi(a|s)\Big| da \\
& = \frac{2\epsilon \gamma}{(1-\gamma)^2} \E_{s \sim d^\pi} \tv \big (\tilde\pi(\cdot|s) || \pi(\cdot|s) \big ).
\end{align}
\end{proof}

\begin{proof}[Derivation of Theorem 1 from \citet{pmlr-v37-schulman15}]
Let $\epsilon = \max_{s,a} | A^\pi (s,a) |$. We note that $ \gamma |\E_{\tau \sim \tilde \pi|s,a} \sum_{t\ge1} \gamma^{t-1} A^\pi(s_t,a_t)| \le
\frac{2\epsilon\gamma}{1-\gamma}\max_s \tv \big(\tilde\pi(\cdot|s)|| \pi(\cdot|s)\big) $. It follows from: for any $s$ the expected advantage $\int \pi(a|s) A^\pi(s,a) da = 0$, so we have
$|\E_{\tau \sim \tilde \pi|s,a} \sum_{t\ge1} \gamma^{t-1} A^\pi(s_t,a_t)| =
|\E_{\tau \sim \tilde \pi|s,a} \sum_{t\ge1} \gamma^{t-1}  \big (\tilde\pi(a_t|s_t) - \pi(a_t|s_t) \big ) A^\pi(s_t,a_t)| \le 
\E_{\tau \sim \tilde \pi|s,a} \sum_{t\ge1} \gamma^{t-1} |\tilde\pi(a_t|s_t) - \pi(a_t|s_t)| |A^\pi(s_t,a_t)| \le \frac{2\epsilon}{1-\gamma}\max_s D^{TV}(\tilde\pi(\cdot|s)|| \pi(\cdot|s))$ . We follow by
\begin{align}
& \Big |\eta(\tilde\pi)- \eta(\pi) -  L_{\tilde\pi}(\pi) \Big | = \\
& = \frac{1}{1-\gamma} \Big | \E_{s \sim d^{\pi}(s)} \int \big (Q^{\tilde\pi}(s,a) - Q^\pi (s,a) \big ) \big (\tilde\pi(a|s) - \pi(a|s) \big)\Big | da  \\
& = \frac{1}{1-\gamma}\Big| \E_{s \sim d^{\pi}(s)} \int \gamma\Big[ \E_{\tau \sim \tilde\pi |s,a} \sum_{t\ge 1} \gamma^{t-1} A_\pi(s_t, a_t) \Big] \big (\tilde\pi(a|s) - \pi(a|s) \big)\Big| da\\
& = \frac{1}{1-\gamma}\Big| \E_{s \sim d^{\pi}(s)} \frac{2\epsilon\gamma}{1-\gamma} \max_s \tv \big(\tilde\pi(\cdot|s) || \pi(\cdot||s) \big) \big (\tilde\pi(a|s) - \pi(a|s) \big)\Big| \\
& \le \frac{4\epsilon \gamma}{(1-\gamma)^2} \max_s\tv \big(\tilde\pi(\cdot|s) || \pi(\cdot||s) \big)^2.
\end{align}
\end{proof}

\begin{proof}[Derivation of first inequality from Theorem 2 from \citet{NIPS2017_6974}]
We introduce the following notation target policy $\tilde\pi$, policy gathering the current batch of data $\pi$ and policy providing off-policy data $\beta$.
Denote vector $f^{\pi, \tilde\pi}$ as a vector with components $f^{\pi, \tilde\pi }(s) = \E_{a \sim \tilde \pi (\cdot|s)} A^\pi (s,a)$. Also, denote $ f_w^{\pi, \tilde\pi}$ parametric approximation to $f^{\pi, \tilde\pi}$ with coordinates $f^{\pi, \tilde\pi}(s) =  \E_{a \sim \tilde \pi (\cdot|s)} A^\pi_w(s,a)$. By using \reflemma{lem:mpie}:
\begin{align}
    &(1-\gamma)\big (\eta(\tilde\pi)-\eta(\pi) \big ) = \langle d^{\pi}, f^{\pi, \tilde\pi} \rangle + \E_{s \sim d^{\pi}} \int \big (Q^{\tilde\pi}(s,a) - Q^{\pi}(s,a)\big )\big (\tilde\pi(a|s) - \pi(a|s) \big ) da.
\end{align}

We can use the following representation for $\langle d^{\pi}, f^{\pi, \tilde\pi} \rangle$:
\begin{align}
    \langle d^{\pi}, f^{\pi, \tilde\pi} \rangle  = 
     (1-\alpha) \langle d^{\pi}, f^{\pi, \tilde\pi} \rangle +
     \alpha \langle d^{\pi}, f^{\pi, \tilde\pi} - f_w^{\pi, \tilde\pi} \rangle + \alpha \langle d^{\beta}, f_w^{\pi, \tilde\pi} \rangle + \alpha \langle d^{\pi} - d^{\beta}, f_w^{\pi, \tilde\pi} \rangle.
\end{align}

Following \cite{NIPS2017_6974} define $\tilde L_{\pi, \beta}(\tilde\pi): = (1-\alpha) \frac{1}{1-\gamma}\E_{s \sim d^{\pi}, a \sim \tilde\pi (\cdot|s)} A^\pi (s,a) + 
\alpha \frac{1}{1-\gamma}\E_{s \sim d^{\beta}, a \sim \tilde\pi (\cdot|s)}  A_w^\pi (s,a)$ and we denote $\delta:= \max_{s,a} |A^{\pi}(s,a) - 
\tilde A_w^{\pi}(s,a)|$, $\zeta := \max_s |f^{\pi, \tilde\pi}(s)|$ and $\epsilon := \max_s |f_w^{\pi, \tilde\pi}(s)|$. We then have that
\begin{align}
    &(1-\gamma)\big (\eta(\tilde\pi)-\eta(\pi) \big )-  \tilde L_{\pi, \beta}(\tilde\pi) \big)  =  \alpha \langle d^\pi, f^{\pi, \tilde\pi} - f_w^{\pi, \tilde\pi} \rangle + \alpha \langle d^\pi - d^\beta, f_w^{\pi, \tilde\pi} \rangle\nonumber\\  
    &+ \E_{s \sim d^\pi} \int \big (Q^{\tilde\pi}(s,a) - Q^{\pi}(s,a) \big) \big (\tilde\pi(a|s) - \pi(a|s) \big ) da.
\end{align}
and
\begin{align}
     & \eta(\tilde\pi)-\eta(\pi) - \tilde L_{\pi, \beta}(\tilde\pi) = \frac{1}{1-\gamma}
     \Big [ \alpha \langle d^{\pi}, f_w^{\pi, \tilde\pi} - f^{\pi, \tilde\pi} \rangle + \alpha \langle d^{\pi} - d^{\beta}, f^{\pi, \tilde\pi} \rangle \nonumber \\
     &+ \E_{s \sim d^{\pi}} \int \big(Q^{\tilde\pi}(s,a) - Q^{\pi}(s,a) \big ) \big (\tilde\pi(a|s) - \pi(a|s) \big )  da \Big ].
\end{align}

We now bound separate terms: $|\langle d^{\pi}, f_w^{\pi, \tilde\pi} - f^{\pi, \tilde\pi} \rangle | \le \delta$ and $|\E_{s \sim d^{\pi}} \int (Q^{\tilde\pi}(s,a) - Q^{\pi}(s,a) \big ) \big (\tilde\pi(a|s) - \pi(a|s) \big ) da| \le \zeta \E_{s \sim d^{\pi}} \int  |\tilde\pi(a|s) - \pi(a|s) \big | da $. By applying Holder inequality to $\langle d^{\pi} - d^{\beta}, f^{\pi, \tilde\pi} \rangle$ we obtain $|\langle d^{\pi} - d^{\beta}, f^{\pi, \tilde\pi} \rangle| \le || d^{\pi} - d^{\beta}||_{1} ||f^{\pi, \tilde\pi} ||_{\infty}$. By using Lemma 3 from Appendix in \citet{pmlr-v70-achiam17a} we get $|| d^{\pi} - d^{\beta}||_{1} \le \frac{2\gamma}{1-\gamma} \E_{s \sim d^\pi} D^{TV}(\pi(\cdot|s)||\beta(\cdot|s))$. As a result $|\langle d^{\pi} - d^{\beta}, f^{\pi, \tilde\pi} \rangle| \le \frac{2\gamma \zeta}{1-\gamma} \E_{s \sim d^\pi} D^{TV}(\pi(\cdot|s)||\beta(\cdot|s))$. Combining these inequalities yields:
\begin{align}
    &\Big | \eta(\tilde\pi)-\eta(\pi) - \tilde L_{\pi, \beta}(\tilde\pi) \Big | \le 
    \frac{1}{1-\gamma} \Big [ \alpha \delta  
    + \frac{2 \alpha \epsilon \gamma}{1-\gamma}\E_{s \sim d^{\pi}} D^{TV}(\tilde\pi(\cdot|s)|| \pi (\cdot|s)) 
    + \frac{2 \gamma \zeta }{1-\gamma} \E_{s \sim d^{\pi}} D^{TV}(\pi(\cdot|s)|| \beta (\cdot|s)) \Big ].
\end{align}

We can then use the inequality $\E_{s \sim d^\pi} D^{TV} (\pi(\cdot|s)||\tilde\pi (\cdot|s)) \le \sqrt{\frac{1}{2}\E_{s \sim d^\pi} D^{KL}(\pi(\cdot|s)||\tilde\pi (\cdot|s))}$ to obtain
\begin{align}
    &\Big | \eta(\tilde\pi)-\eta(\pi) - \tilde L_{\pi, \beta}(\tilde\pi) \Big | \le 
    \frac{1}{1-\gamma} \Big [ \alpha \delta  
    +  \frac{1}{\sqrt 2} \frac{2 \alpha \epsilon \gamma}{1-\gamma} \sqrt{\E_{s \sim d^{\pi_1}} D^{KL}(\tilde\pi(\cdot|s)|| \pi (\cdot|s))} +  \frac{2 \gamma \zeta}{\sqrt 2} \frac{1}{1-\gamma} \sqrt{ \E_{s \sim d^{\pi_1}} D^{KL}(\pi(\cdot|s)||, \pi_2 (\cdot|s)) }\Big ].
\end{align}

Using $\E_{s \sim d^\beta } D^{KL} (\pi(\cdot|s)||\tilde\pi (\cdot|s)) \le \max_s D^{KL} (\pi(\cdot|s)||\tilde\pi (\cdot|s))$ for any policy $\beta$ and upper bounding $\frac{1}{\sqrt 2} < 1$ we obtain the bound from \citet{NIPS2017_6974}:
\begin{align*}
    &\Big | \eta(\tilde\pi)-\eta(\pi) - \tilde L_{\pi, \beta}(\tilde\pi) \Big | \le 
    \frac{1}{1-\gamma} \Big [ \alpha \delta  
    +  \frac{2 \alpha \epsilon \gamma}{1-\gamma} \sqrt{\max_s D^{KL}(\tilde\pi(\cdot|s)|| \pi (\cdot|s))} + \frac{2 \gamma \zeta }{1-\gamma} \sqrt{ \max_s D^{KL}(\pi(\cdot|s)|| \beta (\cdot|s)) }\Big ].
\end{align*}
\end{proof}

\begin{proof}[Derivation of second inequality from Theorem 2 from \citet{NIPS2017_6974}]
To prove equality two from Theorem 2 we define $\tilde L_{\pi, \beta}^{CV}(\tilde\pi): = (1-\alpha) \frac{1}{1-\gamma}\E_{s \sim d^{\pi}, a \sim \tilde\pi (\cdot|s)} (A^\pi (s,a)-A^\pi_w (s,a)) + \frac{1}{1-\gamma}\E_{s \sim d^{\beta}, a \sim \tilde\pi(\cdot|s)}  A_w^\pi (s,a)$.
We can use the following representation for $\langle d^{\pi}, f^{\pi, \tilde\pi} \rangle$:
\begin{align}
    \langle d^{\pi}, f^{\pi, \tilde\pi} \rangle  = 
     (1-\alpha) \langle d^{\pi}, f^{\pi, \tilde\pi} - f_w^{\pi, \tilde\pi}\rangle +
     \langle d^{\beta}, f_w^{\pi, \tilde\pi} \rangle
     + \alpha \langle d^{\pi}, f^{\pi, \tilde\pi} - f_w^{\pi, \tilde\pi} \rangle
     +  \langle d^{\pi} - d^{\beta}, f_w^{\pi, \tilde\pi} \rangle.
\end{align}

We can bound separate terms as follows: $ |\langle d^{\pi}, f^{\pi, \tilde\pi} - f_w^{\pi, \tilde\pi}\rangle| \le \delta $, we use again $|\langle d^{\pi} - d^{\beta}, f^{\pi, \tilde\pi} \rangle| \le \frac{2\gamma \zeta}{1-\gamma} \E_{s \sim d^\pi} D^{TV}(\pi(\cdot|s)||\beta(\cdot|s))$. By applying \reflemma{lem:mpie} it follows that:
\begin{align}
    \eta(\tilde\pi)-\eta(\pi) - \tilde L^{CV}_{\pi, \beta}(\tilde\pi) = \frac{1}{1-\gamma}
     \Big [ \alpha \langle d^{\pi}, f^{\pi, \tilde\pi} - f_w^{\pi, \tilde\pi} \rangle
     +  \langle d^{\pi} - d^{\beta}, f_w^{\pi, \tilde\pi} \rangle
     + \E_{s \sim d^{\pi}} \int \big (Q^{\tilde\pi}(s,a) - Q^{\pi}(s,a) \big ) \big (\tilde\pi(a|s) - \pi(a|s) \big ) da \Big ].
\end{align}

As previously we can use $|\E_{s \sim d^{\pi}} \int (Q^{\tilde\pi}(s,a) - Q^{\pi}(s,a) \big ) \big (\tilde\pi(a|s) - \pi(a|s) \big ) da| \le \frac{2 \gamma \zeta }{1-\gamma} \sqrt{ \max_s D^{KL}(\tilde\pi(\cdot|s)|| \pi (\cdot|s)) }$ and $| \langle d^{\pi} - d^{\beta}, f_w^{\pi, \tilde\pi} \rangle| \le \frac{2 \epsilon \gamma}{1-\gamma} \sqrt{\max_s D^{KL}(\tilde\pi(\cdot|s)|| \beta (\cdot|s))}$ to obtain:
\begin{align}
    &\Big | \eta(\tilde\pi)-\eta(\pi) - \tilde L^{CV}_{\pi, \beta}(\tilde\pi) \Big | \le \frac{1}{1-\gamma} \Big [ \alpha \delta
    +  \frac{2 \epsilon \gamma}{1-\gamma} \sqrt{\max_s D^{KL}(\tilde\pi(\cdot|s)|| \pi (\cdot|s))} + \frac{2 \gamma \zeta }{1-\gamma} \sqrt{ \max_s D^{KL}(\pi(\cdot|s)|| \beta (\cdot|s)) }\Big ].
\end{align}
\end{proof}

\subsection{Experimental setup}

To ensure meaningful comparison, we alter only the PPO policy optimization objective, switching accordingly to Table 1 and keeping any other hyperparameters or parts of the experimental setup unchanged. To parametrize the policy, we use two layer hidden network with tanh activations outputing the mean of Gaussian distribution over actions. The policy standard deviation is parametrized, but state independent. We use two hidden layer neural network for the value function which we learn by minimising the square loss of predicted values with empirical returns. We use $2048$ transitions to perform the policy update. 
\\
\\
To optimize policies we use ADAM optimizer \citep{kingma2014adam} with learning rate set to $3 \cdot 10^{-4}$ and $\epsilon=10^{-5}$, keeping other default Adam parameters. For every policy update we perform $320$ minibatch optimization steps with the minibatch size of $64$. We use the clip range of $0.2$. For GAE critic, we use standard values of discount $\gamma=0.99$ and lambda factor $\lambda = 0.95$. We decay the clip range and learning rate linearly with the passed time steps from the initial values to zeros. We do not use entropy exploration bonus in the course of experiments. We do not perform any environment specific tuning of hyperparameters; we keep hyperparameters fixed during the experimentation. We use only one actor to gather the data required for the policy update. For Humanoid Mujoco experiment we use slightly lower clip value of $0.1$ and constant learning rate schedule. We normalize the state by the moving average of mean and standard deviation.
\\
\\
While performing experiments on Mujoco environments we use standard PPO parameters as listed above. Using environment specific parameters allows to obtain better results, however in this experiment we are aiming to test the robustness of the algorithm.
\\
\\
We report the moving average return over last $200$ episodes as a function of interactions with simulator. We average learning curves over $10$ seeds for Mujoco Humanoid, $20$ seeds for other Mujoco environments and $8$ seeds for Roboschool Humanoid. We based our implementation on publicly available OpenAI Baselines code \citep{baselines}. We open source the code used to run the experiments: \url{https://github.com/marctom/POTAIS}.

\subsection{Further experimental results}

We report comparison of the learning curves for additional settings of $\alpha_t$ and numerical values obtained during experimentation. As discussed setting $\alpha_t \ne (0, \ldots, 0, 1)$ provides visible improvements on several environments. Introducing further non-zero components provides gains in results for first and second component but not for third introduced component.

\begin{table}[ht!]
\tiny
\centering
\begin{tabular}{cccccccc}
\toprule
$\alpha_t$ &     Standup &              Ant &           Hopper &         Humanoid &        Swimmer &         Walker2d &      HalfCheetah\\
\midrule
1.0, 1.0, 1.0      &   $97751.61\pm3035.67$ &   $3240.09\pm178.5$ &  $2306.52\pm202.03$ &  $4593.64\pm174.61$ &   $109.85\pm6.94$ &   $3492.56\pm274.9$ &  $1747.97\pm414.25$ \\
0.0, 1.0, 1.0      &   $97620.58\pm2632.96$ &  $3482.96\pm271.26$ &  $2241.08\pm176.55$ &   $4744.86\pm96.31$ &   $107.67\pm9.02$ &   $3648.88\pm246.6$ &  $1869.14\pm428.66$ \\
0.0, 0.0, 1.0      &   $84469.91\pm1715.18$ &  $2833.46\pm445.97$ &  $2171.21\pm242.25$ &   $3241.8\pm226.38$ &   $98.94\pm11.83$ &  $3814.08\pm245.67$ &  $1764.34\pm310.65$ \\
0.0, 0.0, 0.5      &   $75367.84\pm2344.67$ &   $569.62\pm122.73$ &  $2205.87\pm246.24$ &  $1709.43\pm262.64$ &  $103.85\pm10.69$ &  $2527.38\pm479.23$ &  $2446.52\pm602.89$ \\
0.0, 0.5, 1.0      &   $89588.09\pm1885.68$ &  $3268.42\pm347.93$ &   $2188.41\pm215.1$ &  $4427.96\pm288.42$ &   $113.23\pm6.56$ &   $3850.1\pm224.68$ &  $2044.55\pm484.02$ \\
0.5, 0.5, 1.0      &   $96514.29\pm3524.44$ &   $3529.4\pm192.12$ &   $2112.9\pm214.03$ &  $4741.39\pm235.57$ &   $110.29\pm9.41$ &  $3728.18\pm218.05$ &  $2175.29\pm510.66$ \\
0.5, 0.5, 0.5      &  $100208.41\pm4092.96$ &  $3204.55\pm208.36$ &  $2281.64\pm199.03$ &  $4618.06\pm165.57$ &    $105.0\pm6.96$ &  $3639.47\pm276.92$ &  $2184.15\pm481.82$ \\
0.0, 0.0, 0.25     &   $73079.89\pm2007.63$ &     $46.07\pm23.06$ &  $1375.32\pm348.86$ &     $745.95\pm36.9$ &   $91.84\pm16.29$ &   $566.21\pm116.52$ &  $2688.29\pm533.54$ \\
0.25, 0.5, 1.0     &   $93785.21\pm2962.43$ &  $3340.99\pm385.43$ &  $2111.29\pm214.28$ &  $4657.25\pm176.76$ &    $113.9\pm6.66$ &  $3824.08\pm246.22$ &   $1604.46\pm261.9$ \\
0.25, 0.25, 0.25   &   $89172.19\pm2127.32$ &  $3396.38\pm121.09$ &  $2220.28\pm201.53$ &  $3741.41\pm141.29$ &   $108.15\pm3.44$ &  $3423.39\pm217.89$ &  $2477.07\pm551.84$ \\
0.5, 0.5, 0.5, 1.0 &  $100912.25\pm3823.57$ &  $3509.53\pm161.82$ &  $2124.92\pm135.99$ &  $4810.98\pm200.86$ &   $113.66\pm5.76$ &  $3871.68\pm187.82$ &  $1820.97\pm416.67$ \\
\midrule
\end{tabular}

\caption{Mean scores obtained for Mujoco experiments with standard deviation error bars averaged over $20$ seeds.}
\end{table}

\begin{table}[ht!]
\centering
\small
\begin{tabular}{cc}
\toprule
$\alpha_t$ &    $\frac{\eta(\pi_{AISPO})}{\eta(\pi_{PPO})}$  \\
\midrule
$(0, \ldots, 0.0 , 0.0, 0.25 )$    &  $0.73\pm0.12$ \\
$(0, \ldots, 0.0 ,0.0, 0.0, 0.5 )$      &  $0.87\pm0.12$ \\
$(0, \ldots, 0.0 ,0.0, 0.0, 1.0 )$      &    $1.0\pm0.1$ \\
$(0, \ldots, 0.0 ,0.0, 0.5, 1.0  )$     &    $1.1\pm0.1$ \\
$(0, \ldots, 0.0 ,0.0, 1.0, 1.0  )$     &  $1.13\pm0.09$ \\
$(0, \ldots, 0.0 ,0.5, 0.5, 1.0  )$     &   $\bf {1.14\pm0.1}$ \\
$(0, \ldots, 0.0 ,0.5, 0.5, 0.5   )$    &   $\bf {1.14\pm0.1}$ \\
$(0, \ldots, 0.0 ,0.25, 0.5, 1.0   )$   &   $1.1\pm0.09$ \\
$(0, \ldots, 0.0 ,0.25, 0.25, 0.25 )$   &  $1.11\pm0.09$ \\
$(0, \ldots, 0.0 ,1.0, 1.0, 1.0   )$    &  $1.11\pm0.09$ \\
$(0, \ldots, 0.0 ,0.5, 0.5, 0.5, 1.0 )$ &  $1.13\pm0.08$ \\
\bottomrule

\end{tabular}
\caption{Mean relative improvement over PPO for different choices of $\alpha_t$ with standard deviation error bars averaged over $20$ seeds.}
\end{table}

We present learning curves for different choices of $\alpha_t$.

\begin{figure}[ht!]
\hspace*{-0.25cm}                                                           
\includegraphics[width=170mm,height=70mm]{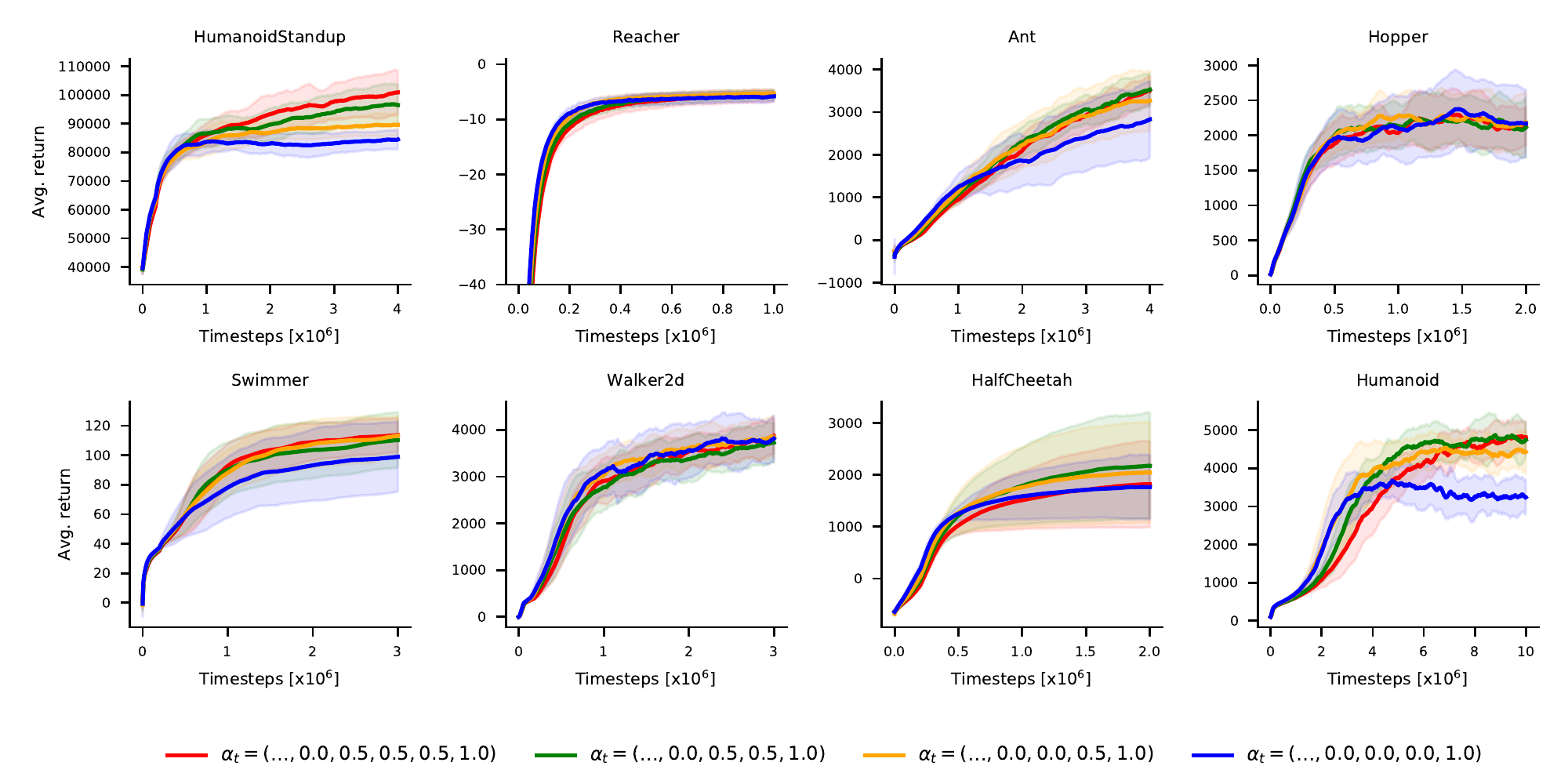}
\caption{The effect of adding non-zero components $\beta$ when $\alpha_t=(0, \ldots, 0.0, \beta, \ldots,\beta, 1)$. Adding first and second non-zero component provides visible improvements in performance but adding third component does not improve mean performance across environments. Blue curve corresponds to PPO.}
\end{figure}

\begin{figure}[ht!]
\hspace*{-0.25cm}                                                           
\includegraphics[width=170mm,height=80mm]{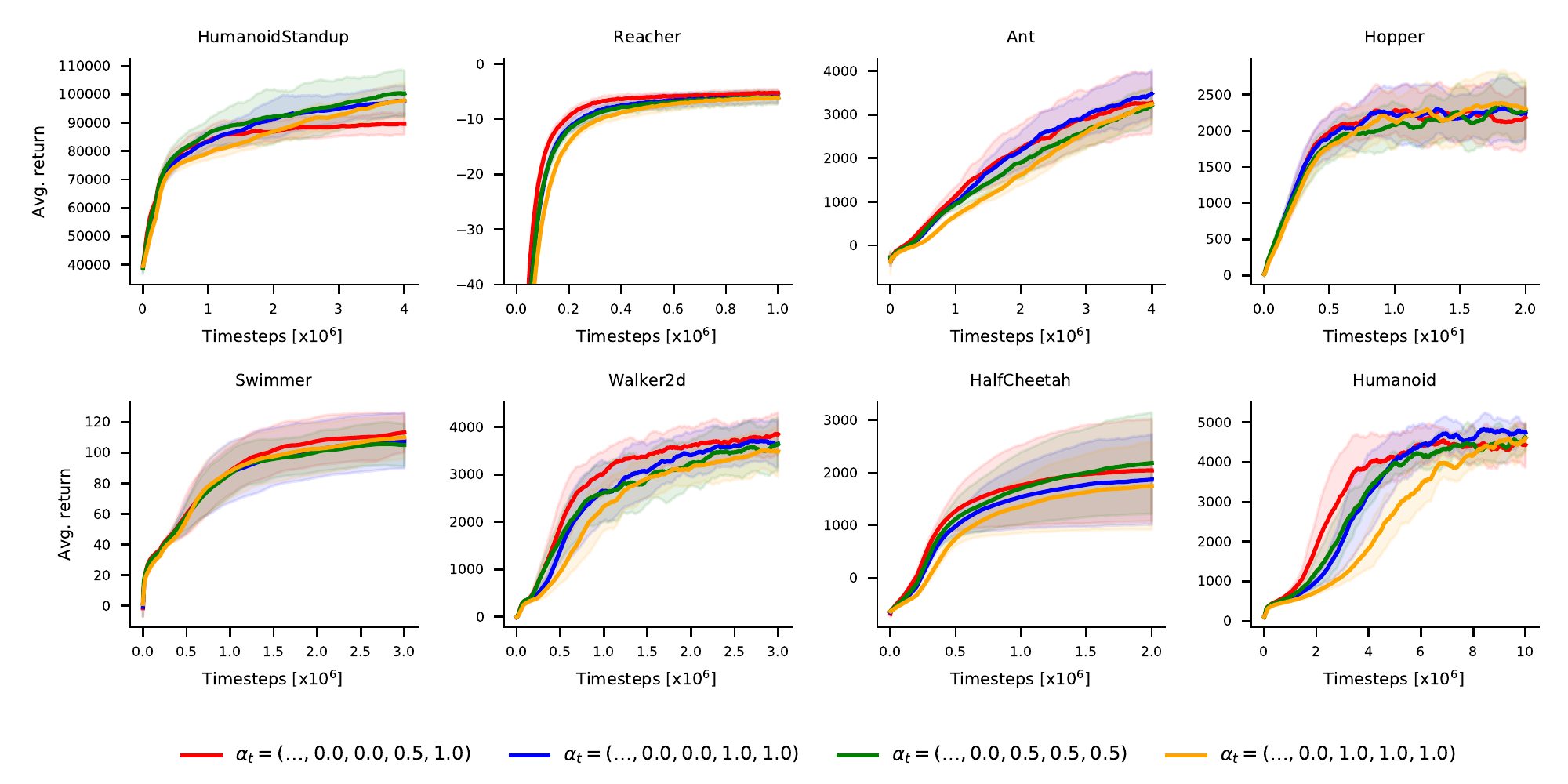}
\caption{The effect of changing values of components $\beta_1$, $\beta_2$ and $\beta_3$ when $\alpha_t=(0, \ldots, 0,\beta_1 ,\beta_2 ,\beta_3)$. All choices improve over PPO. }
\end{figure}

\begin{figure}[ht!]
\hspace*{-0.25cm}                                                           
\includegraphics[width=170mm,height=80mm]{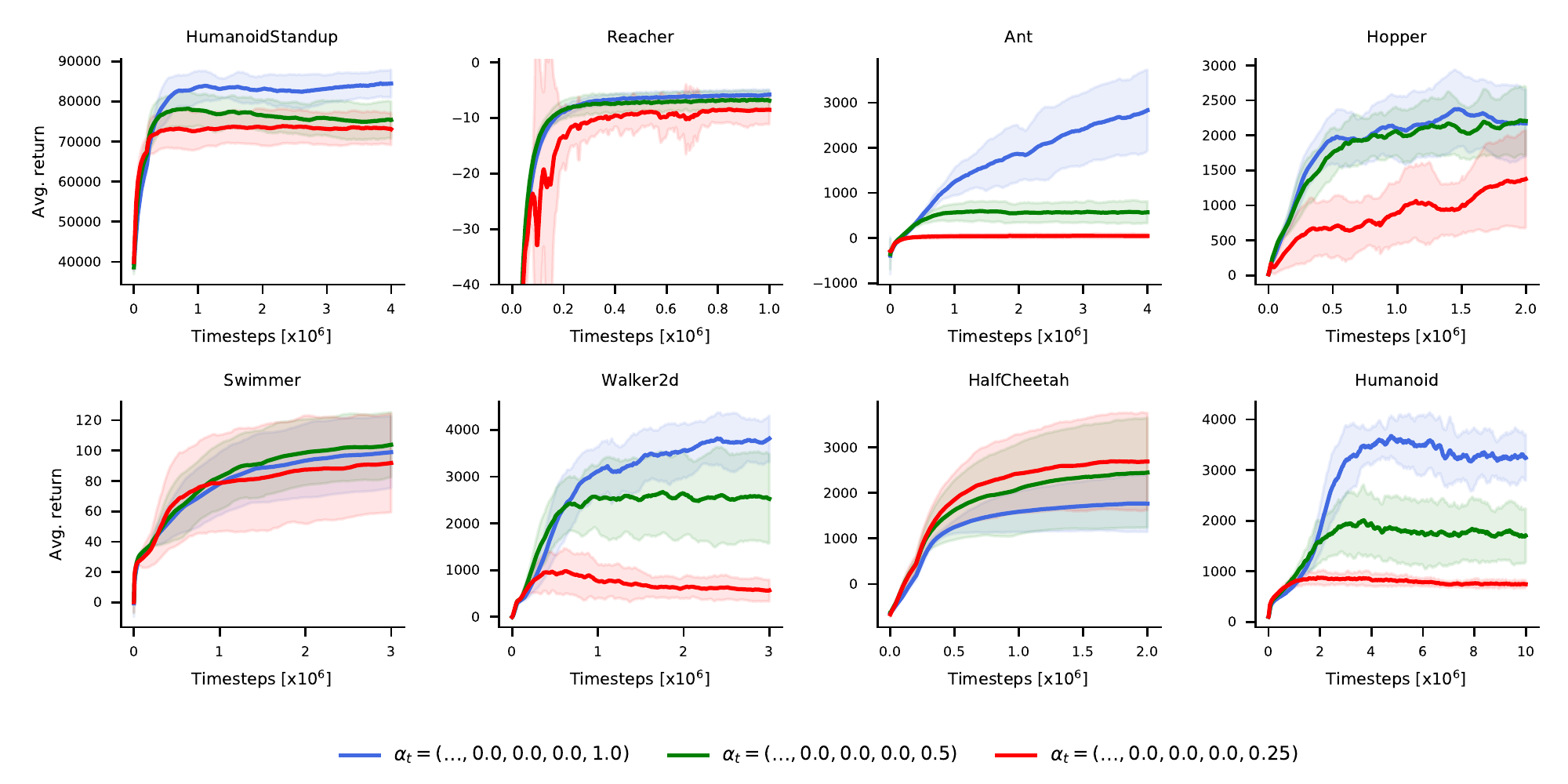}
\caption{Comparison of $\alpha_t=(0, \ldots, 0.0, \beta)$. Smoothing the last component reduces the variance but provides excessive bias. This degrades the performance on most environments. Interestingly the performance on HalfCheetah environment improves.}
\end{figure}

\end{document}